\documentclass{article}

\usepackage{iclr2027_conference,times}
\iclrfinalcopy

\usepackage[utf8]{inputenc} 
\usepackage[T1]{fontenc}    
\usepackage{hyperref}       
\usepackage{url}            
\usepackage{booktabs}       
\usepackage{amsfonts}       
\usepackage{amsmath}
\usepackage{amssymb}
\usepackage{amsthm}
\usepackage{enumitem}
\usepackage{nicefrac}       
\usepackage{microtype}      
\usepackage{xcolor}         
\usepackage{algorithm}
\usepackage{algorithmic}
\usepackage{graphicx}
\usepackage{tikz}
\usetikzlibrary{arrows.meta, positioning}
\usepackage{wrapfig}
\renewcommand{\headrulewidth}{0pt}

\usepackage{xspace}
\usepackage{comment}

\title{Beyond State-Wise Mirror Descent: Offline Policy Optimization with Parametric Policies}

\author{%
Xiang Li \qquad Yuheng Zhang \qquad Nan Jiang \\
University of Illinois Urbana-Champaign \\
\texttt{\{xiang306,yuhengz2,nanjiang\}@illinois.edu}
}
\newcommand{\tmax}{\textup{max}}
\newcommand{\tmin}{\textup{min}}
\newcommand{\cp}{\textup{cp}}
\newcommand{\Reg}{\textup{Reg}}
\newcommand{\KL}{D_\textup{KL}}
\newcommand{\bias}{\textup{CFA}}

\newcommand{\stat}{\textup{stat}}
\newcommand{\err}{\textup{err}}

\newcommand{\CVaR}{\textup{CVaR}}

\newcommand{\Comp}{\textup{Comp}}
\newcommand{\feat}{\textup{feat}}
\newcommand{\spn}{\textup{span}}

\newcommand{\cS}{\mathcal{S}}
\newcommand{\cM}{\mathcal{M}}
\newcommand{\cA}{\mathcal{A}}
\newcommand{\cT}{\mathcal{T}}
\newcommand{\cF}{\mathcal{F}}
\newcommand{\cE}{\mathcal{E}}
\newcommand{\cO}{\mathcal{O}}
\newcommand{\cN}{\mathcal{N}}
\newcommand{\cL}{\mathcal{L}}
\newcommand{\cD}{\mathcal{D}}
\newcommand{\cW}{\mathcal{W}}

\newcommand{\RR}{\mathbb{R}}
\newcommand{\EE}{\mathbb{E}}

\newcommand{\dd}{\mathsf{d}}
\newcommand{\PP}{\mathsf{P}}
\newcommand{\UU}{\mathsf{U}}
\newcommand{\CC}{\mathsf{C}}
\newcommand{\sign}{\mathsf{s}}

\newcommand{\simiid}{\mathrel{\overset{\mathrm{i.i.d.}}{\sim}}}

\theoremstyle{plain}

\newtheorem{theorem}{Theorem}
\newtheorem{proposition}[theorem]{Proposition}
\newtheorem{lemma}[theorem]{Lemma}

\newtheorem{definition}{Definition}
\newtheorem{assumption}{Assumption}

\begin{document}

\maketitle

\fancyhead[L]{}

\begin{abstract}
  We investigate the theoretical aspects of offline reinforcement learning (RL) under general function approximation. While prior works \citep[e.g.,][]{xie2021bellman} have established the theoretical foundations of learning a good policy from offline data via pessimism, existing algorithms that are computationally tractable (often in an oracle-efficient sense), such as PSPI, only apply to finite and small action spaces. Moreover, these algorithms rely on \emph{state-wise} mirror descent and require actors to be implicitly induced from the critic functions, failing to accommodate standalone policy parameterization which is ubiquitous in practice. 
In this work, we address these limitations and extend the theoretical guarantees to parameterized policy classes over large or continuous action spaces. When extending mirror descent to parameterized policies, we identify \emph{contextual coupling} as the core difficulty, and show how connecting mirror descent to natural policy gradient leads to novel analyses, guarantees, and algorithmic insights, including a surprising unification between offline RL and imitation learning.
\end{abstract}

\section{Introduction}
\label{sec:introduction}

Learning a good policy from historical data, a.k.a.~offline reinforcement learning (RL), is an important paradigm for bringing RL to real-life domains \citep{levine2020offline}. While the statistical aspects of offline RL theory are relatively well-understood \citep{jiang2025offline}, the information-theoretic algorithms that achieve strong statistical guarantees under general conditions are often not immediately computationally tractable, and addressing the computational challenge typically requires additional assumptions or sacrificing some of the statistical generality.
In particular, while the analyses often make mild assumptions on the \textit{critic} (i.e., value function approximator), strong restrictions are often imposed on the \textit{actor} (i.e., policy approximator), limiting the theory's applicability scope and practical relevance \citep{xie2021bellman,cheng2022adversarially}.

As a canonical example, consider PSPI by \citet{xie2021bellman} (Section~\ref{sec:pspi}), an actor-critic style algorithm that achieves strong guarantees when critics are modeled via general function approximation. PSPI assumes an oracle that computes a pessimistic critic $f_k$ for policy $\pi_k$, and employs \textit{mirror descent} (MD) in the actor update: 
\begin{align} \label{eq:md}
\pi_{k+1}(a\mid s)\propto\pi_k(a\mid s)\exp\bigl(\eta f_k(s,a)\bigr).
\end{align}
This brings two issues. First, the guarantee of PSPI depends on the log-cardinality of the action space and thus does not apply to arbitrarily large (including continuous) action spaces that are ubiquitous in control problems (e.g., Gaussian policies in robotics). Secondly, MD update in Eq.\eqref{eq:md} is \textit{state-wise}, i.e., the action distribution in each state is updated independently. This means that the actor (policy) is implicitly induced from the critics ($f_k$) and cannot have its own standalone parameterization (e.g., an actor network that is separate from the critic network). 
These limitations not only create a disconnection between theory and practice, but also indicate a gap in theory itself: standalone policy parameterization over continuous actions is not an issue at all when we only pursue statistical guarantees and ignore computational feasibility \citep[Section 3]{xie2021bellman}, but difficulties arise when we aim for computational efficiency. 

In this work, we address this gap by systematically studying variants of policy optimization methods and their statistical/computational properties, under the assumption that a  standard pessimistic critic oracle is given (Assumption~\ref{ass:oracle}). Our contributions are 3-fold:

\begin{itemize}[leftmargin=*]
    \item \textbf{PSPI with general action spaces.}
    In Section~\ref{sec:pspi}, we revisit the PSPI algorithm to disentangle critic- and actor-side errors via performance decomposition, and present a mild extension of the analysis to arbitrarily large (and even continuous) action spaces. 

    \item \textbf{Contextual coupling: a hardness result.}
    In Section~\ref{sec:challenge}, we consider standalone policy parameterization and identify an underlying challenge we coin \emph{contextual coupling}. We formalize a general class of \emph{local} updates and show that they can incur constant per-step regret even with exact critics and sufficient data coverage. This negative result highlights the difficulty of reconciling policy improvements at different states under a shared parameterization, motivating the need for additional assumptions.

    \item \textbf{Actor updates with finite-sample guarantees.}
    In Section~\ref{sec:algorithm}, we introduce a general regret decomposition via compatible function approximation (CFA), where CFA serves as an additional assumption that circumvents the previous hardness result. The regret decomposition provides a guiding principle for designing effective local updates. Based on this principle, we develop two statistically and computationally efficient actor updates: a \emph{least-square regression} update-rule (LSPU) closely related to natural policy gradient, and a \emph{distributionally robust} update-rule (DRPU) that leverages importance weighting and can be more robust to actor-critic incompatibility. Surprisingly, when the offline distribution coincides with that of the comparator policy, our DRPU method reveals a connection to behavior cloning, providing an interesting unification between offline RL and imitation learning.  
\end{itemize}



\section{Preliminary}
\label{sec:preliminary}

\paragraph{Markov Decision Processes.} 

We formulate RL in infinite-horizon discounted Markov Decision Processes (MDPs), specified by $(\cS,\cA,P,R,\gamma, d_0)$, where $\cS$ is the state space, $\cA$ is the action space, $P:\cS\times\cA\to\Delta(\cS)$ is the transition kernel, $R:\cS\times\cA\to[0,R_\tmax]$ is the reward function, $\gamma\in[0,1)$ is the discount factor, and $d_0\in\Delta(\cS)$ is the initial state distribution. We mostly consider finite and discrete $\cS$ and $\cA$ with possibly arbitrarily large cardinalities, and our results can be easily extended to continuous action spaces as exemplified in Appendix~\ref{app:continuous-setting}.

A (stochastic) policy $\pi:\cS\to\Delta(\cA)$ specifies a decision-making strategy. Its expected discounted return is denoted by $J(\pi)=\EE_\pi[\sum_{t=0}^\infty\gamma^tr_t]$, with $s_0 \sim d_0$, $a_t\sim\pi(\cdot|s_t)$, $r_t=R(s_t,a_t)$, and $s_{t+1}\sim P(\cdot|s_t,a_t)$. The $Q$-function of $\pi$ is defined by  $Q^\pi(s,a)=\EE_\pi[\sum_{t=0}^\infty\gamma^tr_t|s_0=s,a_0=a]$, which is the unique fixed point of its Bellman operator $\mathcal{T}^\pi$: $\forall f\in\RR^{\cS\times\cA}$, $(\cT^\pi f)(s,a)=R(s,a)+\gamma\EE_{s'\sim P(\cdot|s,a)}[f(s',\pi)]$, where $f(s',\pi)=\EE_{a'\sim\pi(\cdot|s')}[f(s',a')]$. $V^\pi(s)=Q^\pi(s,\pi)$ is the state value function and $A^\pi(s,a)=Q^\pi(s,a)-V^\pi(s)$ is the advantage function. Let $d^\pi$ denote the discounted state-action occupancy distribution of policy $\pi$, i.e., $d^\pi(s,a)=(1-\gamma)\sum_{t=0}^\infty\gamma^t d_t^\pi(s,a)$, where $d_t^\pi$ is the marginal distribution of $(s_t, a_t)$.\footnote{With a slight abuse of notation, we also use $d^\pi$ to denote the discounted state occupancy of policy $\pi$ (after marginalizing out the actions). That is, $d^\pi(s)=\sum_{a\in\cA}d^\pi(s,a)$.}



\paragraph{Offline RL.} 

In offline RL, we learn from pre-collected data without environment interaction. For conceptual clarity, we assume access to two datasets: a \textit{critic dataset} with transition tuples $(s,a,r,s')$ for fitting the critic, and an \textit{actor dataset} consisting of only $(s,a)$ pairs for actor updates. 

As our focus is on the actor, we do \textbf{not} explicitly characterize the critic dataset, since its properties will later be encapsulated in an oracle assumption (Assumption~\ref{ass:oracle}). For the actor dataset, we assume access to $N$ i.i.d. $(s,a)$ pairs drawn from a data distribution $d^D$, which might be arbitrary as long as it satisfies appropriate coverage condition. In practice, they may be extracted from the critic dataset, but our formulation also allows additional data sources, such as expert annotations, which are considered in hybrid setting of imitation learning and offline RL \citep{mao2023offline,yang2023hundreds}.

We adopt the following standard assumption on data coverage, known as \emph{density coverage} \citep{jiang2025offline}. The theoretical guarantees in offline RL typically state that we can compete with any comparator policy $\pi_\cp$ that is well-covered by data (i.e., $\CC$ is small).

%

\begin{assumption}[Actor-Side Data Coverage]\label{ass:data}
Assume $\sup_{s,a}|d^{\pi_\cp}(s,a)/d^D(s,a)|\le \CC < \infty$.
\end{assumption}

To deal with large state and action spaces, we are in the function approximation regime, modeling $Q$-functions using a function class $\cF=\{f:\cS\times\cA\to[0,V_\tmax]\}$, where $V_\tmax=R_\tmax/(1-\gamma)$. Typical choices of $\cF$ include linear classes or deep neural networks. In the literature, $\cF$ generally needs to satisfy expressivity conditions such as realizability or Bellman completeness \citep{chen2019information}, which will also be implicitly captured in Assumption~\ref{ass:oracle}.

\paragraph{Math notation.} We use $A \lesssim B$ or $A = \cO(B)$ to denote $A \leq cB$ for some constant $c > 0$; $A=\Omega(B)$ is equivalent to $B=\cO(A)$. For a function $f:\mathcal{X}\to\RR$ for finite $\mathcal{X}$, its supremum norm is $\|f\|_\infty=\max_{x\in\mathcal{X}}|f(x)|$. The KL divergence is $\KL(p\|q)=\sum_{x\in \mathcal{X}}p(x)\log\tfrac{p(x)}{q(x)}$ for two distributions with probability mass functions $p$ and $q$.
\section{Pessimistic Soft Policy Iteration as State-wise Mirror Descent}
\label{sec:pspi}

\citet{xie2021bellman} proposed Pessimistic Soft Policy Iteration (PSPI), an actor-critic method for offline policy optimization (see Algorithm~\ref{alg:pspi}). At a high level, PSPI proceeds iteratively: at each iteration $k$, the critic produces a pessimistic estimate $f_k$ of the value $Q^{\pi_k}$, and the actor performs a soft improvement step based on $f_k$. The final output is a (trajectory-level) uniform mixture of the intermediate policies $\pi_1,\ldots,\pi_K$.

\begin{algorithm}[t]
   \caption{Pessimistic Soft Policy Iteration (PSPI)} \citep{xie2021bellman}
   \label{alg:pspi}
\begin{algorithmic}
   \STATE {\bfseries Input:} number of iterations $K$, learning rate $\eta$
   \STATE Initialize policy $\pi_1$ as an arbitrary policy (e.g., uniform over $\cA$)
   \FOR{$k=1,2,\ldots,K$}
      \STATE {\bf Critic:} compute $\pi_k$'s pessimistic $Q$-function $f_k$ using oracle $\cO$ satisfying Assumption~\ref{ass:oracle}
      \STATE {\bf Actor:} update policy by Eq.~\eqref{eq:md}, i.e.,
      $
      \pi_{k+1}(a\mid s)\propto\pi_k(a\mid s)\exp\bigl(\eta f_k(s,a)\bigr)$
   \ENDFOR
   \STATE {\bfseries Output:} $\hat{\pi}=\text{Unif}[\pi_{1:K}]$
\end{algorithmic}
\end{algorithm}

Since our ultimate goal is to bound the suboptimality gap $J(\pi_\cp)-J(\hat{\pi})$ between some comparator policy $\pi_\cp$ and the output policy $\hat{\pi}$, a key component in the analysis is to decompose it via the generalized performance difference lemma (Lemma~\ref{lemma:PDL}): $J(\pi_\cp)-J(\hat{\pi})=$
\begin{align*}
\frac{1}{1-\gamma}\Biggl(
\underbrace{
\frac{1}{K}\sum_{k=1}^K
\EE_{s\sim d^{\pi_\cp}}
[f_k(s,\pi_\cp)-f_k(s,\pi_k)]
}_{\text{actor-side error (Eq.~\eqref{eq:regret}), \textbf{focus of this paper}}}
+
\underbrace{
\frac{1}{K}\sum_{k=1}^K
\Bigl(\EE_{d^{\pi_\cp}}[\cT^{\pi_k}f_k-f_k]+\EE_{d^{\pi_k}}[f_k - \cT^{\pi_k}f_k]
\Bigr)
}_{\text{critic-side Bellman error, handled by Assumption~\ref{ass:oracle}}}
\Biggr).
\end{align*}
The first term corresponds to actor-side optimization error, capturing how well $\pi_1,\ldots,\pi_K$ optimizes against $\pi_\cp$ measured by the pessimistic $Q$-function; the remaining terms are critic-side errors that measure the violation of the Bellman equations. Such decomposition cleanly separates the difficulty of policy optimization (i.e., actor update) from that of value estimation.

To isolate the actor-side difficulty, we abstract the critic as a pessimistic oracle $\cO$, as reflected in Assumption~\ref{ass:oracle}, which directly controls the latter two Bellman error terms. Such control is central to offline RL theory and essentially an intermediate guarantee achieved in many existing works under standard assumptions \citep{jin2020provably,xie2021bellman,cheng2022adversarially}. 
More detailed discussions and algorithmic realizations of this oracle, e.g., via Bellman error minimization or marginalized importance sampling, are provided in Appendix~\ref{app:oracle}.\footnote{Here we treat Assumption~\ref{ass:oracle} as deterministic, though its realizations are typically high-probability statements. A fully end-to-end result that combines our actor guarantees and the critic guarantees will need to split the failure probability among the actor and the critic analyses (say $\delta/2$ each), but that will only result in a constant blow-up. }
\begin{assumption}[Pessimistic Oracle]\label{ass:oracle}
We have an efficient oracle $\cO$ that, given any policy $\pi$ as input, it outputs a function $f$ that satisfies the following two conditions:
\begin{enumerate}[leftmargin=*]
    \item (\emph{Pessimism}) The function $f$ is a pessimistic estimation of $Q^\pi$ (up to some tolerance $\epsilon_r\geq0$), i.e., $J_f(\pi)-J(\pi)\leq\epsilon_r/(1-\gamma)$, where $J_f(\pi)=\EE_{s\sim d_0}[f(s,\pi)]$.
    \item (\emph{Bounded ``Transferred" Bellman error}) The Bellman error under $d^{\pi_\cp}$ is bounded by some $\epsilon_b> 0$. That is, $\EE_{d^{\pi_\cp}}[\cT^\pi f-f]\leq\epsilon_b$.
\end{enumerate}
\end{assumption}

Under Assumption~\ref{ass:oracle}, $\epsilon_r+\epsilon_b$ naturally controls the latter two terms in the suboptimality decomposition. Thus, bounding $J(\pi_\cp)-J(\hat{\pi})$ reduces to controlling the \emph{actor-side} error, i.e., the first term in the decomposition, which is a notion of \textbf{\emph{regret}}. 
Our focus for the remainder of the paper is thus to design algorithms that produce $\pi_1, \ldots, \pi_K$ with well-bounded regret: 
\begin{equation}\label{eq:regret}
\frac{\Reg_K}{K}:=\frac{1}{K}\sum_{k=1}^K\EE_{s\sim d^{\pi_\cp}}\Bigl[f_k(s,\pi_\cp)-f_k(s,\pi_k)\Bigr].
\end{equation}

\paragraph{State-wise regret control.}
To control the regret, \citet{xie2021bellman} solve this online optimization problem in a \emph{state-wise} manner: $\pi_{k+1}(a\mid s)\propto\pi_k(a\mid s)\exp(\eta f_k(s,a))$, regardless of the underlying distribution $d^{\pi_\cp}\in\Delta(\cS)$.
%
Equivalently, 
for each state, the policy update in Eq.~\eqref{eq:md} corresponds to performing \textbf{\emph{mirror descent}} (MD) with KL regularizer \citep{hazan2016introduction}.

This online learning perspective leads to the regret guarantee established for PSPI in~\citep{xie2021bellman}, which, however, depends on $|\cA|$, the cardinality of $\cA$.\footnote{Apart from the regret's dependence on $|\cA|$, the computation required by PSPI might also seem to require enumeration over $\cA$ since we need to sample from the softmax policy in Eq.\eqref{eq:md}. This is, however, not always the case, as one can use tools such as rejection sampling to avoid enumeration, and structured settings (e.g., LQR) admit closed-form expressions for the policy that can be efficiently sampled from; see Appendix~\ref{app:continuous-ex} for details.
} We next show that their result can be directly extended to remove the explicit $|\cA|$ dependence and thus handle arbitrarily large action spaces; we state the result below for finite and discrete $\cA$ to be consistent with the setup of the rest of the paper.

\begin{theorem}[Regret Bound of Algorithm~\ref{alg:pspi}]\label{thm:pspi}
Assume that for each iteration $k$, the critic $f_k$ is bounded in $[0,V_\tmax]$, and that $\KL(\pi_\cp\|\pi_1)<\infty$. Then by choosing $\eta=\sqrt{8\KL(\pi_\cp\|\pi_1)/(KV_\tmax^2)}$, PSPI iterates satisfy
\[
\frac{\Reg_K}{K}\le V_\tmax\sqrt{\frac{\KL(\pi_\cp\|\pi_1)}{2K}},
\]
where $\KL(\pi_\cp\|\pi)=\EE_{s\sim d^{\pi_\cp}}[\KL(\pi_\cp(\cdot|s)\|\pi(\cdot|s))]$ is the expected KL divergence under $d^{\pi_\cp}$.
\end{theorem}

The result is a mild extension of \citet[Theorem 4.1]{xie2021bellman}, where we keep $\KL(\pi_\cp\|\pi_1)$ instead of upper-bounding it with $\log|\cA|$ as in \citet{xie2021bellman}. This allows the term to be well bounded for some structured policies, even when $\cA$ is unbounded in cardinality or even continuous. In Appendix~\ref{app:continuous-setting}, we show that Theorem~\ref{thm:pspi} can be extended straightforwardly to continuous action space,\footnote{Results in subsequent sections can also be extended to continuous action spaces in a similar manner, which we omit and stick to the discrete setup for readability.} and for standard forms of continuous policies such as Gaussian policies, or structural action spaces like convex action spaces, the KL term remains well bounded; see Appendix~\ref{app:continuous-ex} for details.

\section{Contextual Coupling in Parameterized Policy Optimization}
\label{sec:challenge}

The state-wise mirror descent of PSPI in Section~\ref{sec:pspi} does not have its own policy parameterization; rather, it is performing policy optimization within a softmax policy class implicitly induced from the critic class $\cF$: $\{\pi(\cdot|s)\propto\pi_1(\cdot|s)\exp\bigl(\textstyle{\sum_{i=1}^k f_i(s,\cdot)}\bigr):k\in[K],f_{1:k}\in\cF\}$. Instead, we now consider policy optimization over a standalone policy class parameterized by some $\theta\in\RR^\dd$, which is ubiquitous in practice:
\[
\Pi_\theta := \{\pi_\theta:\theta\in\RR^\dd\}.
\]
As standard, we consider algorithms that make gradient-type updates to the actor: 
\begin{equation}\label{eq:update}
    \theta_{k+1}=\theta_k+\eta v_k,
\end{equation}
where $\eta$ is the step-size, $v_k\in\RR^\dd$ is an update vector that may depend on the current critic $f_k$, policy $\pi_k$, and offline data. The goal is to produce a sequence of $\pi_k = \pi_{\theta_k}$ with well-controlled regret in Eq.\eqref{eq:regret}, which immediately implies that $\textrm{Unif}[\pi_{1:K}]$ (as in PSPI) is a mixture policy with desirable guarantees we can output.  
We assume standard differentiability and smoothness assumptions widely adopted in the theoretical analyses of policy-gradient algorithms  \citep{agarwal2021theory}. 

\begin{assumption}[Policy Class]\label{ass:policy}
Let $\pi_\theta$ be a differentiable policy parameterized by $\theta\in\RR^\dd$. 
For all $(s,a)\in\cS\times\cA$, the log-probability is $G$-Lipschitz and $\beta$-smooth with respect to a chosen norm $\|\cdot\|$ (and its dual norm $\|\cdot\|_*$).
\[
|\log\pi_\theta(a|s)-\log\pi_{\theta'}(a|s)|\le G\|\theta-\theta'\|,\quad 
\|\nabla\log\pi_\theta(a|s)-\nabla\log\pi_{\theta'}(a|s)\|_*\le \beta\|\theta-\theta'\|.
\]
\end{assumption}
\vspace{0.2pt}



\paragraph{Contextual coupling.} We now explain the difficulty caused by standalone parameterization that makes the analysis of PSPI in Section~\ref{sec:pspi} no longer applicable, a problem we coin \textit{contextual coupling}. Recall that PSPI 
runs mirror descent for each state and enjoys regret guarantee independently across states. Importantly, this allows us to swap the expectation and the average operators:
\begin{align*}
\frac{\Reg_K}{K}=\frac{1}{K}\sum_{k=1}^K\EE_{s\sim d^{\pi_\cp}}\Bigl[f_k(s,\pi_\cp)-f_k(s,\pi_k)\Bigr] =\EE_{s\sim d^{\pi_\cp}}\Bigl[\underbrace{ \textstyle  \frac{1}{K}\sum_{k=1}^K \left(f_k(s,\pi_\cp)-f_k(s,\pi_k)\right)}_{\textrm{Per-state regret}}\Bigr].
\end{align*}
State-wise mirror descent controls the per-state regret for each $s$, which leads to bounded regret averaged over $d^{\pi_\cp}$ \textbf{even when the algorithm has no access to the distribution $d^{\pi_\cp}$}. 

With a standalone policy parameterization $\pi_\theta$, however, this analysis breaks down as the policies $\pi_k$ are \textit{coupled} across states: we can no longer set the distribution of $\pi_k(\cdot|s)$ for each state $s$ separately, but can only do so by updating the shared parameter $\theta$. An update direction that optimizes $f_k(s, \cdot)$ for one state $s$ can be harmful for another, causing conflicts. 
We would know how to balance such a trade-off between states if we had access to $d^{\pi_\cp}$, which is unfortunately not available. 
The only distribution we have access to is $d^D$, so one might hope that we control the expected regret under $d^D$ and transfer it to $d^{\pi_\cp}$ through coverage assumptions (e.g., Assumption~\ref{ass:data}). Unfortunately, 
the integrand $f_k(s,\pi_\cp)-f_k(s,\pi_k)$ is not guaranteed to be non-negative, so errors controlled under $d^D$ do not necessarily translate to $d^{\pi_\cp}$ even with well-covered data.

\newcommand{\dataup}{local update\xspace}
\newcommand{\dataups}{local updates\xspace}
\newcommand{\Dataup}{Local update\xspace}
\newcommand{\DataUp}{Local Update\xspace}
\newcommand{\DataUps}{Local Updates\xspace}

\paragraph{Hardness of \dataups.} We formalize the above intuition as a hardness result, showing that controlling actor regret in the offline setting is impossible without further assumptions. The hardness result applies to a wide range of algorithms which we call \emph{\dataups}---as they make updates based on standard statistics (e.g., the critic values) on the ``local'' distribution $d^D$---even with an \textit{infinite amount of data} (thus the update rules discussed below have access to population statistics).  

\begin{definition}[\DataUps]\label{def:local-update}
$\theta_{k+1}=\theta_k+\eta v_k$ is a \emph{\dataup} using population statistics if its update direction $v_k$ satisfies $v_k\sim\UU_k(\cdot\mid \theta_{1:k},\PP_k)$, where $\UU_k$ is a probability kernel common to all problem instances, and 
$
\PP_k:=\mathrm{Law}_{s\sim d^D}\left(d^D(\cdot\mid s),\theta\mapsto\pi_\theta(\cdot\mid s),\{f_j(s,\cdot)\}_{j\le k}\right).
$
In particular, $\PP_k$ is the joint distribution of the variables in the parentheses.\footnote{The variables include the \textit{functional} $\theta \mapsto\pi_\theta(\cdot\mid s)$, which allows for the computation of $\nabla_\theta \log\pi$ and the likelihood ratio $\pi_{\theta}(a\mid s) / \pi_{\theta'}(a\mid s)$ needed for KL divergence. 
See Appendix~\ref{app:challenge-alg} for algorithms that satisfy and violate Definition~\ref{def:local-update}.} 
\end{definition}

Definition~\ref{def:local-update} captures a natural class of actor updates that only depend on 
population statistics (such as policy scores $\nabla_\theta\log\pi$, KL divergences, and critic values) under the offline data distribution $d^D$. 
As verified in Appendix~\ref{app:challenge-alg}, a natural generalization of state-wise mirror descent to the contextual case, as well as the algorithms we will introduce in later sections, satisfies Definition~\ref{def:local-update} naturally. Using this definition, we establish the following hardness result:  


\begin{proposition}[Failure of \DataUps]\label{thm:hardness}

There exist two contextual bandits instance $\cM_+$ and $\cM_-$ sharing the same log-linear policy class $\Pi_\theta$ satisfying Assumption~\ref{ass:policy} and offline data satisfying Assumption~\ref{ass:data} (with respect to comparator policy $\pi_{\cp,+},\pi_{\cp,-}\in\Pi_\theta$), such that any local update in Definition~\ref{def:local-update} satisfies $\max_{\cM\in\{\cM_+,\cM_-\}}\mathbb{E}_\cM[\Reg_K/K]=\Omega(1)$ for all $K\ge 1$, even with infinite data and exact critics $f_k=Q^{\pi_k}$.

\end{proposition}

\paragraph{Implications.} One potential approach to circumvent the hardness is to move beyond \dataups. For example, one can use importance sampling to simulate ``non-local'' on-policy distributions $d^{\pi_k}$ from the local $d^D$, which violates Definition~\ref{def:local-update} (see Proposition~\ref{prop:def-non-update}) and breaks the hardness construction in Proposition~\ref{thm:hardness}. However, this either requires trajectory-wise importance sampling which has high variance \citep{precup2000eligibility}, or invokes additional machinery to estimate the marginalized importance weights for distribution correction \citep{nachum2019dualdice}, which often requires all-policy coverage, defeating the purpose of pessimistic critics; see \citet[Assumption C]{pmlr-v151-huang22a} and \citet{huang2024occupancybased}'s discussion below their Theorem 4.1. 

For the rest of this paper, we investigate the alternative of staying with \dataups and ask: \textit{what additional regularity assumptions enable regret guarantees for \dataups, and how to design update rules accordingly?}

\section{Constructing Unified \DataUps}\label{sec:algorithm}

We now turn to developing \dataups with provable guarantees. Motivated by the contextual coupling challenge identified in Section~\ref{sec:challenge}, we seek to understand what properties of the update direction $v_k$ in Eq.~\eqref{eq:update} are sufficient to overcome this difficulty and control comparator regret.

Let $A_k(s,a):=f_k(s,a)-f_k(s,\pi_k)$ denote the critic-induced advantage at round $k$. Under Assumption~\ref{ass:policy}, the change in log-policy caused by Eq.~\eqref{eq:update} admits a first-order approximation whose leading term is $v_k^\top\nabla_\theta\log\pi_k(a\mid s)$, where $\pi_k$ denotes $\pi_{\theta_k}$. This suggests comparing the critic advantage $A_k$ with the linearized policy change induced by $v_k$, which is the notion of \textbf{\emph{compatible function approximation}} (CFA) \citep{sutton1999policy, kakade2001natural}. Therefore we define the CFA error by
\begin{equation}\label{eq:cfa}
\err_k=\EE_{(s,a)\sim d^{\pi_\cp}}\left[A_k(s,a)-v_k^\top\nabla_\theta\log\pi_k(a\mid s)\right].
\end{equation}
The following lemma identifies the CFA error as the potentially non-vanishing term in the comparator-regret bound for general local updates. The proof is deferred to Appendix~\ref{app:challenge-decomposition}.

\begin{lemma}[Regret Decomposition Lemma]\label{lemma:NPG-regret}
Under Assumption~\ref{ass:policy}, consider the update in Eq.~\eqref{eq:update} with $\|v_k\|\le B_L$ for all $k$. Then, choosing $\eta=\sqrt{2\KL(\pi_\cp\|\pi_1)/(\beta K B_L^2)}$,\footnote{Our later results (Theorems~\ref{thm:npg-guarantee} and \ref{thm:dro-guarantee}) will use this learning rate, which requires knowledge of $\pi_\cp$ in a mild manner as the regret guarantee is robust to learning-rate misspecification; see the end of Appendix~\ref{app:challenge-decomposition}.}  we have
\[
\frac{\Reg_K}{K}\leq B_L\sqrt{\frac{2\beta\cdot \KL(\pi_\cp\|\pi_1)}{K}}+  \frac{1}{K}\sum_{k=1}^K\err_k.
\]
\end{lemma}

Lemma~\ref{lemma:NPG-regret} reveals how actor-critic compatibility enters the regret analysis: apart from the optimization term, \textbf{controlling comparator regret reduces to constructing update directions with small CFA error}. This provides a principled criterion for designing effective local updates. Related analyses in the on-policy setting appear in \citep{f870a19c-3a42-3013-ac0e-1c1db320f051, agarwal2021theory}; to our knowledge, our work is the first to connect this perspective to distribution shift in offline RL.

The main challenge, however, is that the CFA error is evaluated under the unknown distribution $d^{\pi_\cp}$, whereas local updates can leverage information from offline data $d^D$. In the following subsections, we develop two  efficient approaches for controlling the CFA error: one
based on least-squares regression and the other on distributionally robust optimization.

\subsection{Least Square Policy Update}
\label{sec:npg}
We first consider controlling the CFA error through a least-squares regression objective. Viewing $\nabla\log\pi_k(a\mid s)$ as \emph{features} and $A_k(s,a)$ as the \emph{regression target}, this observation naturally leads to a noiseless linear regression formulation for constructing the update $v_k$, which we term \textbf{\emph{least square policy update}} (LSPU). Specifically, we define the least-square loss at round $k$ as
\[
L_k(v):=\EE_{(s,a)\sim d^D}\bigl[(A_k(s,a)-v^\top\nabla\log\pi_k(a\mid s))^2\bigr],
\]
where the expectation is taken on offline data distribution $d^D$. 
In each round $k$, the actor first computes the advantage function $A_k$, and then obtains the update $v_k$ by minimizing $L_k$ using samples, within the norm constraint $\|v_k\|\le B_L$. In particular, given $\{(s^{(i)},a^{(i)})\}_{i=1}^N \simiid d^D$ from the actor dataset, the constrained regression problem is given by:
\begin{equation}\label{eq:ols-solution}
v_k
=\underset{v:\|v\|\le B_L}{\arg\min}\frac1N\sum_{i=1}^N\Bigl(A_k(s^{(i)},a^{(i)})-v^\top\nabla\log\pi_k(a^{(i)}|s^{(i)})\Bigr)^2.
\end{equation}

\paragraph{Error transfer through the coverage condition.}\label{sec:npg-coverage}

A key question is why the regression objective can be formed under the data distribution $d^D$, even though the CFA error is defined under $d^{\pi_\cp}$. The  reason is that the squared loss $L_k$ is always \emph{non-negative} (c.f. negative results in Section~\ref{sec:challenge}), allowing us to do distribution transfer via the coverage condition:
\[
\EE_{d^{\pi_\cp}}\bigl[(A_k-v^\top\nabla\log\pi_k)^2\bigr]\leq \CC\cdot\EE_{d^D}\bigl[(A_k-v^\top\nabla\log\pi_k)^2\bigr].
\]
Moreover, since this is a linear regression problem, this density coverage (known as \emph{concentrability coefficient}) can be further improved in the compatible case, i.e., when $A_k=w^\top\nabla\log\pi_k$ for some $w\in\RR^\dd$. Let the covariance matrix be $\Sigma_D^{\pi}=\EE_{d^D}[\nabla\log\pi\nabla\log\pi^\top]$. In this~case, we can sharpen the density coverage $\CC$ (Assumption~\ref{ass:data}) with the notion of \emph{feature coverage} \citep{jiang2025offline}, though~we omit such an improvement in our main theorems for readability (see Appendix~\ref{app:dro-w2} for details): $\CC_\feat=\max_{k\in[K]}\ \EE_{d^{\pi_\cp}}[\nabla\log\pi_k]^\top(\Sigma_D^{\pi_k})^{-1}\EE_{d^{\pi_\cp}}[\nabla\log\pi_k]$.

\paragraph{Relation to NPG.}
This least-square interpretation of policy update can be viewed as a form of the \emph{natural policy gradient} (NPG) method under function approximation, which is well-studied in the on-policy setting \citep{kakade2001natural,peters2008natural,agarwal2021theory}. That said, our formulation differs from canonical on-policy (or off-policy) NPG since LSPU is computed on the offline data distribution $d^D$ without any importance-weight correction. Such design choice follows directly from the decomposition in Lemma~\ref{lemma:NPG-regret} and tailors specifically to offline RL.

\paragraph{Actor-critic incompatibility.}
Even with an accurate critic $f_k$, the linear regression formulation above need \textbf{not} be well-specified. In general, the target $A_k$ does not necessarily lie in the linear span of the features $\nabla\log\pi_k$. Hence we are in the \emph{agnostic learning} setting and any regression-based update may inevitably incur an approximation error, even with infinite data. We capture such misspecification via a quantity $\epsilon_\bias$ defined below.

\begin{assumption}[LSPU Approximation Error]
\label{ass:npg}
Define $v_k^*$ as the best linear approximator of the least-square loss $L_k$. Assume that $v_k^*$ satisfies the norm constraint $\|v_k^*\|\le B_L$\footnote{Assumption~\ref{ass:npg} requires that the global minimizer $v_k^*$, which admits a closed-form solution, must satisfy $\|v_k^*\|\le B_L$. This can be achieved by setting $B_L$ as a uniform upper bound $G V_\tmax/\lambda_\tmin$, where $\lambda_\tmin$ denotes the smallest eigenvalue of the feature covariance matrices $\{\Sigma_D^{\pi_k}\}_{k=1}^K$.}, and its loss is uniformly bounded by some $\epsilon_\bias\ge 0$, i.e., $L(v_k^*)=\min_{v\in\RR^\dd}L_k(v)\le\epsilon_\bias$.
\end{assumption}

The quantity $\epsilon_\bias$ uniformly measures the extent to which $\nabla\log\pi_k$ can compatibly express $A_k$ at each round. To ensure that $\epsilon_\bias$ is small, the actor and critic must align with each other. Thus we refer to $\epsilon_\bias$ as \textbf{\emph{actor-critic incompatibility}}, which vanishes in the compatible case, i.e. when $A_k = w^\top \nabla \log \pi_k$ for some $w \in \RR^\dd$. 
For instance, the soft policy class used in PSPI (see Section~\ref{sec:pspi}) satisfies $\epsilon_\bias=0$. In linear function approximation, $\epsilon_\bias = 0$ when both $\Pi_\theta$ and $\cF$ share the same feature representation (see Appendix~\ref{app:nobias}); for neural networks, shared actor-critic representations are also prevalent in deep RL practice \citep{mnih2016asynchronous}. Our finite-sample guarantee for LSPU below includes a constant term that depends on $\epsilon_\bias$, which is inevitable given Proposition~\ref{thm:hardness}, since LSPU is a form of local update in Definition~\ref{def:local-update}; in fact, the hardness construction is incompatible and will be ruled out if we assume compatibility $\epsilon_\bias=0$.

We now establish a finite-sample regret guarantee for LSPU
; the proof can be found in Appendix~\ref{app:lspu-proof}.

\begin{theorem}[Main Theorem for LSPU]
\label{thm:npg-guarantee}
Let $\lambda_\tmin>0$ denote (a deterministic lower bound of) the smallest eigenvalue of $\{\Sigma_D^{\pi_k}\}_{k=1}^K$. Under Assumptions~\ref{ass:data},~\ref{ass:policy},~\ref{ass:npg}, and regularity conditions for linear regression (Assumption~\ref{ass:linear-regression}), policy updates $\{v_k\}$ computed in Eq.~\ref{eq:ols-solution} achieve the following regret bound with probability at least $1-\delta$, given that $N$ satisfies the burn-in condition in Eq.\eqref{eq:lspu-burnin}:
\[
\frac{\Reg_K}{K}\lesssim B_L\sqrt{\frac{\beta\KL(\pi_\cp\|\pi_1)}{K}}+\sqrt{\CC\epsilon_\bias}
+G\sqrt{\frac{\CC\epsilon_\bias\Comp_L}{\lambda_\tmin N}}+\textup{h.o.t.},
\]
where $\Comp_L$ is some complexity measure on the policy and the value class (defined in Eq.~\eqref{eq:complexity}), and ``h.o.t.'' stands for ``higher-order terms'' of order $\cO(1/N)$.
\end{theorem}

Theorem~\ref{thm:npg-guarantee} shows that the regret bound of LSPU admits a 3-fold decomposition: an \emph{optimization} term due to the update $\theta_{k+1}=\theta_k+\eta v_k$, a \emph{bias} (approximation error) term determined by the actor-critic incompatibility, and a \emph{statistical estimation} term that decays at a rate of $\cO(\sqrt{\CC/N})$. In the well-specified case $\epsilon_\bias=0$, exact empirical regression has zero statistical error on the uniform covariance event, so $K=\cO(1/\varepsilon^2)$ rounds suffice once Eq.~\eqref{eq:lspu-burnin} holds (Appendix~\ref{app:lspu-proof}).

\subsection{Distributionally Robust Policy Update}
\label{sec:dro}

In Section~\ref{sec:npg}, we use a squared surrogate to control the CFA error, which enjoys favorable statistical and computational properties of linear regression. However, it is inherently a relaxation: the squared loss does not directly correspond to the linear form of $\err_k$, and may be loose when the approximation error is highly heterogeneous across the state-action space. This naturally raises a question: can we control the linear CFA error directly
? Related ideas have been explored in marginalized importance sampling  \citep{liu2018breaking, uehara2020minimax,xie2020q}.

One viable approach is to adopt a \emph{distributionally robust optimization} (DRO) perspective. We can express the error under the unknown $d^{\pi_\cp}$ as an importance-weighted expectation under $d^D$:
\begin{equation}\label{eq:robust-transfer}
\left|\EE_{d^{\pi_\cp}}[A_k-v_k^\top\nabla\log\pi_k]\right|=\left|\EE_{d^D}\left[w^*(s,a)(A_k-v_k^\top\nabla\log\pi_k)\right]\right|,
\end{equation}
where $w^*$ is some \emph{correction weight} that allows distribution transfer between $d^{\pi_\cp}$ and $d^D$, 
such as $w^*(s,a)=d^{\pi_\cp}(s,a)/d^D(s,a)$, though alternative forms exist which we will discuss later. 
Since $w^*$ is unknown, 
we introduce a robust loss that considers the worst-case weight over a weight class $\cW$:
\begin{align*}
\ell_k(v):=\max_{w\in\cW}\left|\EE_{d^D}\left[w(A_k-v^\top\nabla\log\pi_k)\right]\right|.
\end{align*}
The weight class $\cW$ may be specified explicitly or induced implicitly as a nonparametric space. As long as it satisfies \emph{realizability} (i.e., $w^*\in\cW$), we have $\err_k \leq \ell_k(v_k)$, and minimizing $\ell_k$ suffices. We refer to the update that leverages DRO as the \textbf{\emph{distributionally robust policy update}} (DRPU).

\paragraph{The $\cW_\infty$ class and computation.}

A natural instantiation of the correction weight $w^*$ is the density ratio $d^{\pi_\cp}/d^D$ which directly allows the distribution transfer in Eq.~\eqref{eq:robust-transfer}, i.e., $w^*(s,a)=d^{\pi_\cp}(s,a)/d^D(s,a)$. Correspondingly, we can consider the following \emph{bounded-density-ratio} class:
\begin{equation}\label{eq:weight}
\cW_\infty=\Bigl\{w:\cS\times\cA\to[0,\CC]:\EE_{(s,a)\sim d^D}[w(s,a)]=1\Bigr\}.
\end{equation}
which includes all valid density ratios due to the normalization constraint. This choice aligns with the density coverage condition (Assumption~\ref{ass:data}). In particular, realizability holds since $\|w^*\|_\infty = \|d^{\pi_\cp}/d^D\|_\infty \leq \CC$. From now on we instantiate $\cW = \cW_\infty$ and keep the notation $\ell_k$ for simplicity.

One key advantage of using $\cW_\infty$ lies in its \emph{computational} properties. Under $\cW_\infty$ class, the robust loss $\ell_k$ admits an equivalent dual representation (Proposition~\ref{prop:cvar}), which is a version of \emph{Conditional Value-at-Risk} (CVaR) objective:
\[
\ell_k(v)
=
\max_{\sign\in\{\pm1\}}
\min_{\tau\in\RR}
\Bigl\{\tau+\CC\ \EE_{d^D}\big[(\sign(A_k-v^\top\nabla\log\pi_k)-\tau)_+\big]\Bigr\}\,
\]
where $\sign\in\{\pm 1\}$ denotes the sign and $(Z)_+ = \max\{0,Z\}$. Using samples $\{(s^{(i)}, a^{(i)})\}_{i=1}^N \simiid d^D$, DRPU solves for the following convex function with respect to $v$:
\begin{equation}\label{eq:cvar-lp}
\hat{\ell}_k(v)=\max_{\sign\in\{\pm1\}}\min_{\tau\in\RR}\Big\{\tau+\frac{\CC}{N}\sum_{i=1}^N \big(\sign (A_k^{(i)}-v^\top\nabla\log\pi_k(a^{(i)}|s^{(i)}))-\tau\big)_+\Big\}.
\end{equation}
Depending on the norm constraint, the above optimization can be reformulated as a linear program (or a second-order convex program), hence admitting efficient numerical algorithms; see Appendix~\ref{app:dro-computation}.


\paragraph{Guarantee.}
As in Section~\ref{sec:npg}, we allow for actor-critic incompatibility in the actor update and characterize it through the following assumption. 
The norm bound is set as $B_L=V_\tmax$ in this section.

\begin{assumption}[DRPU Approximation Error]\label{ass:dro}
Define $\tilde{v}_k^*$ as the best linear approximator of the robust loss $\ell_k$ (within the norm constraint $\|\tilde{v}_k^*\|\le V_\tmax$). Assume its loss is uniformly bounded by some $\widetilde{\epsilon}_\bias\geq 0$ for all $k$, i.e., $\ell_k(\tilde{v}_k^*)=\min_{v:\|v\|\le V_\tmax}\ell_k(v)\leq\widetilde{\epsilon}_\bias$.
\end{assumption}

Like LSPU, DRPU also belongs to the class of local updates in Definition~\ref{def:local-update}. Moreover, its approximation error $\widetilde{\epsilon}_\bias$ is closely related to the bias term $\epsilon_\bias$ associated with LSPU in Assumption~\ref{ass:npg}. In particular, taking $B_L=V_\tmax$, by the Cauchy-Schwarz inequality, one can show that
$\widetilde{\epsilon}_\bias\leq\sqrt{\CC\cdot\epsilon_\bias}$,
a relationship that we will explore further in Section~\ref{sec:comparison}. As a consequence, \textbf{DRPU can be more robust to actor-critic incompatibility}; but running the algorithm also relies on a reasonably tight knowledge of the coverage constant $\CC$, as also discussed in \citet[Section 7]{xie2020q}.

The finite-sample guarantee for this DRPU method under $\cW_\infty$ class is given below. Exhibiting a similar 3-fold structure of Theorem~\ref{thm:npg-guarantee}, it shows that $\err_k$ scales as $\mathcal{O}(\sqrt{\CC/N})$, up to an additional bias term. The proof is deferred to Appendix~\ref{app:dro-proof}.

\begin{theorem}[Main Theorem for DRPU under $\cW_\infty$ Class]
\label{thm:dro-guarantee}
Under Assumptions~\ref{ass:data},~\ref{ass:policy},~\ref{ass:dro} and the finite-cover conditions of Appendix~\ref{app:dro-proof}, policy updates $\{v_k\}$ minimizing $\{\widehat{\ell}_k\}$ in Eq.~\eqref{eq:cvar-lp} achieves the following regret bound with probability at least $1-\delta$:
\[
\frac{\Reg_K}{K}\lesssim V_\tmax\sqrt{\frac{\beta\KL(\pi_\cp\|\pi_1)}{K}}+\widetilde\epsilon_\bias
+V_\tmax(G+1)\sqrt{\frac{\CC\Comp_D}{N}}+\textup{h.o.t.},
\]
where $\Comp_D$ is some complexity measure defined in Eq.~\eqref{eq:dro-complexity}.
\end{theorem}


\paragraph{Alternative weight classes.}

$\cW_\infty$ in Eq.~\eqref{eq:weight} is one instantiation of the weight class. More generally, the DRO framework allows for alternative characterizations of $\cW$, as long as it satisfies the transfer in Eq.~\eqref{eq:robust-transfer}. Different choices of $\cW$ correspond to different structural assumptions on the distribution shift between $d^{\pi_\cp}$ and $d^D$. For instance, in the compatible case, one can leverage the improved coverage notion of feature coverage in Section~\ref{sec:npg} to define the following chi-square weight class:
$
\cW_{\chi^2}=\big\{w:\EE_{(s,a)\sim d^D}[w(s,a)^2]\leq \CC_\feat\bigr\}.
$ 
It suffices to show that such a weight class contains solution to Eq.~\eqref{eq:robust-transfer}. Indeed, at round $k$ consider $w^*(s,a)=\EE_{d^{\pi_\cp}}[\nabla\log\pi_k]^\top(\Sigma_D^{\pi_k})^{-1}\nabla\log\pi_k(a|s)$, which satisfies both Eq.\eqref{eq:robust-transfer} and  $w^*\in\cW_{\chi^2}$; 
see Appendix~\ref{app:dro-w2} for details.

\begin{figure}[t]
    \centering
    \begin{minipage}[t]{0.47\linewidth}
        \vspace{0pt}
        \centering
        \includegraphics[width=\linewidth]{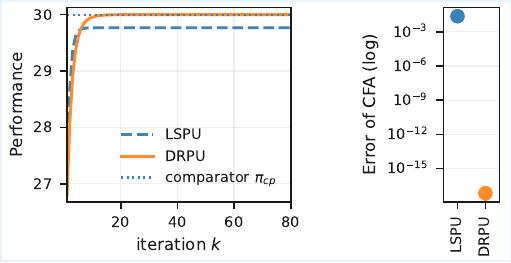}
    \end{minipage}
    \hfill
    \begin{minipage}[t]{0.49\linewidth}
        \vspace{0pt}
        \caption{Comparison between LSPU and DRPU under no-shift setting ($d^D=d^{\pi_\cp}$). 
        \textbf{Left:} performance $J(\pi_k)$ over iterations, where DRPU approaches the comparator performance (\emph{not} optimal), while LSPU plateaus at a worse policy.
        \textbf{Right:} the CFA error, $\err_k$, at iteration $k=80$ on a log scale, showing that DRPU drives the error close to zero, whereas LSPU incurs a larger bias.}
        \label{fig:comparison}
    \end{minipage}
\end{figure}

\subsection{Bias Comparison between LSPU and DRPU}
\label{sec:comparison}

As shown in Theorems~\ref{thm:npg-guarantee} and \ref{thm:dro-guarantee}, the regret bounds are mainly driven by the intrinsic bias terms when optimization and statistical errors vanish (that is, when sample size $N$ and optimization iterations $K$ are sufficiently large). Such terms are defined differently for LSPU (Assumption~\ref{ass:npg}) and DRPU (Assumption~\ref{ass:dro})
. As quantified in Section~\ref{sec:dro}, DRPU has potential advantage in terms of robustness to actor-critic incompatibility than LSPU. In this section, we use a case study to illustrate this behavior when $d^D=d^{\pi_\cp}$. In fact, perhaps surprisingly, \textbf{DRPU can be viewed as behavior cloning in this setting}, providing an interesting unification between offline RL and imitation learning. 



\paragraph{No-shift case.}
To isolate the effect of actor-critic incompatibility, we consider the case where $d^D=d^{\pi_\cp}$. This rules out the difficulty of distribution shift and naturally arises in hybrid settings of imitation learning and offline RL, where we have access to additional expert annotations (i.e., data from $\pi_\cp$). In this case, we can choose the weight class as realizable $\cW=\{\mathbf{1}\}$ since $d^{\pi_\cp}/d^D\equiv 1$. The DRPU then reduces to the following \emph{mean-matching} problem:
\begin{equation}\label{eq:mm}
\min_{v:\|v\|\leq V_\tmax}\left|\EE_{d^{\pi_\cp}}[A_k]-v^\top\EE_{d^{\pi_\cp}}[\nabla\log\pi_k]\right|.
\end{equation}
This formulation highlights a fundamental difference between LSPU and DRPU. While LSPU enforces pointwise control by minimizing a squared loss, DRPU only requires agreement in expectation under a certain distribution $d^{\pi_\cp}$. Therefore, Eq.~\eqref{eq:mm} can be driven to exactly zero under mild conditions, even when the actor and the critic are incompatible. We empirically illustrate this phenomenon using a simple MDP (details in Appendix~\ref{app:experiment}), as shown in Figure~\ref{fig:comparison}. 


The mean-matching update also admits a natural interpretation. Consider the policy update $\theta_{k+1}=\theta_k+\eta v_k$, where $v_k$ solves Eq.~\eqref{eq:mm}. This update can be viewed as performing a steepest descent step\footnote{This holds when $\EE_{d^{\pi_\cp}}[A_k] > 0$. Roughly speaking, this is the typical situation as it is equivalent to $J(\pi_k) \le J(\pi_\cp)$ when the critic is accurate; see  Appendix~\ref{app:bc}.} within $[-V_\tmax,V_\tmax]$ on the expected KL divergence under $d^{\pi_\cp}$, $\KL(\pi_\cp\|\pi_\theta)$, which is exactly the objective of \emph{behavior cloning} (BC). In this sense, the mean-matching view reveals an interesting connection between offline RL and imitation learning at the level of policy updates.


\section{Conclusion}
\label{sec:conclusion}
In this paper, we study offline policy optimization with standalone policy classes. We identify \emph{contextual coupling} as a fundamental challenge and show that a broad class of local updates can suffer constant per-step regret, even with accurate critics and sufficient data coverage.
%
To address this difficulty, we develop a general regret decomposition through compatible function approximation, leading to two actor updates, LSPU and DRPU, with provable guarantees statistically and computationally.

One limitation is that our analysis primarily focuses on explicit stochastic policy classes, such as log-linear policy or Gaussian policy, where log-density-based requirements (e.g., Assumption~\ref{ass:policy}) are well-defined. Extending the theory to deterministic or implicit generative policies (e.g., diffusion policy) remains an important open problem and may require fundamentally different analytical tools, as suggested by recent work on continuous-control learning \citep{ren2024diffusion,simchowitz2025pitfalls}. We view bridging this gap as a promising direction for future research.




\bibliographystyle{iclr2027_conference}
\bibliography{reference}

\newpage
\appendix

\newpage
\onecolumn
\appendix


\section{Related Work}\label{app:rw}
Offline RL studies learning policies from a fixed dataset without further environment interaction~\citep{levine2020offline}. With function approximation, a dominant line of work builds on value-function estimation and approximate dynamic programming~\citep{ernst2005tree,munos2007performance,antos2008learning,munos2008finite,farahmand2010error,chen2019information,xie2020q}. This value-centric paradigm underlies many practical algorithms~\citep{mnih2013playing} as well as much of the accompanying theory, and provides the standard backdrop for understanding statistical and computational aspects in offline RL. See~\citep{jiang2025offline} for a comprehensive survey of offline RL theory.

A central concept in offline RL is \emph{data coverage}, with existing guarantees typically relying on either \emph{all-policy} coverage, which requires sufficient coverage of $d^\pi$ for all policies in a class, or \emph{single-policy} coverage, which only requires coverage of a comparator distribution such as $d^{\pi_\cp}$. Two common approaches to handling single-policy coverage are \emph{behavior regularization}, which keeps learned policies close to the behavior distribution~\citep{fujimoto2019off}, and \emph{pessimism}, which optimizes guaranteed performance under data-consistent models or value functions~\citep{kumar2020conservative,jin2020provably,rashidinejad2021bridging,xie2021bellman,cheng2022adversarially}.

Most prior theoretical treatments of offline RL emphasize critic-side learning and treat the actor as an induced object, for example via greedy or soft-greedy policies with respect to learned value functions. In contrast, a smaller but growing literature studies offline RL from an explicit actor--critic or policy-search perspective. The work most closely related to ours is~\citet{xie2021bellman}, which induces a soft-greedy policy class with respect to pessimistic value estimates (see Section~\ref{sec:pspi}). This enables an elegant state-wise mirror-descent interpretation of the actor update, but does not extend to standalone policy classes. Subsequent work~\citep{cheng2022adversarially} further develops this actor--critic viewpoint under stronger oracle assumptions on the actor (see Definition~4 of~\citet{cheng2022adversarially}) and trains the actor using policy-gradient-style updates based on empirical demonstrations. As we show in this paper, such oracle assumptions are strong, and this formulation can obscure the precise source of actor-side error as well as the role of distribution mismatch.

From a policy-based reinforcement learning perspective, our work also connects to the broader literature on policy gradient (PG)~\citep{williams1992simple} and natural policy gradient (NPG)~\citep{kakade2001natural} methods. On-policy analyses of (natural) policy gradient establish convergence and sample complexity guarantees under smooth parametric policy classes~\citep{fazel2018global,agarwal2021theory,mei2020global,yuan2022linear}, and highlight the importance of distribution shift. In particular,~\citet{agarwal2021theory} elucidates the connection between NPG, compatible function approximation~\citep{sutton1999policy}, and least-squares regression, yielding a regret decomposition into estimation and approximation errors. We build on this policy-optimization viewpoint, but focus on the offline setting, where the distribution mismatch between the data distribution and the visitation distribution of the updated policies fundamentally alters what can be guaranteed. There is also a line of work on off-policy policy gradient methods~\citep{liu2019off,kallus2020statistically}, including extensions to offline data~\citep{zhou2023offline}. These approaches are largely rooted in analyses of on-policy policy gradient with importance weighting, whereas our work studies policy optimization directly in the offline setting and explicitly addresses distribution shift through data coverage.

Finally, continuous action spaces further amplify these challenges. Early work such as~\citet{antos2008learning} extends fitted Q-iteration to continuous actions by explicitly searching over a policy class rather than computing $\arg\max_a Q(s,a)$. Much of the offline RL theory for continuous actions relies on strong structural assumptions, including linear-quadratic regulators, linear MDPs, Gaussian policies, or other restricted dynamics and policy families~\citep{van2007reinforcement,faust2014continuous,haarnoja2018soft,bennett2023low}. In contrast, modern practice often employs flexible standalone policy classes, such as deep neural networks, optimized directly via policy-based methods. This is particularly pronounced in large-scale post-training of language models, where policies are represented by deep networks and updated using PG-type algorithms such as PPO \citep{schulman2017proximal} or GRPO \citep{guo2025deepseek}. These developments further motivate a principled understanding of policy optimization under distribution shift and function approximation. Our results contribute to this direction by characterizing a fundamental obstruction to naive contextual policy optimization in offline RL and by proposing actor updates with provable guarantees under standalone policy classes.

\section{Omitted Details for Section~\ref{sec:pspi}}\label{app:pspi}

\subsection{Discussion on the Critic Oracle in Assumption~\ref{ass:oracle}}
\label{app:oracle}
In Assumption~\ref{ass:oracle}, we propose a critic oracle $\cO$ such that given any policy $\pi$, it finds a function $f\in\cF$ such that $f$ satisfies the following two conditions:
\begin{enumerate}[leftmargin=*]
    \item (\emph{Pessimism}) The function $f$ is a pessimistic estimation of $Q^\pi$ (up to some tolerance $\epsilon_r\geq0$), i.e., $J_f(\pi)-J(\pi)\leq\epsilon_r/(1-\gamma)$ with high probability, where $J_f(\pi)=\EE_{s\sim d_0}[f(s,\pi)]$.
    \item (\emph{Bounded Bellman error}) The Bellman error under $d^{\pi_\cp}$ (comparator policy's occupancy) is bounded by some $\epsilon_b> 0$. That is, $\EE_{d^{\pi_\cp}}[\cT^\pi f-f]\leq\epsilon_b$ with high probability.
\end{enumerate}

The Bellman error telescoping (Lemma~\ref{lemma:BE-telescoping}) translates the pessimism condition to $\EE_{d^{\pi}}[f-\cT^\pi f]=(1-\gamma)(J_f(\pi)-J(\pi))\leq\epsilon_r$. This means the oracle can control the Bellman error $(f-\cT^\pi f)$ under both $d^{\pi_\cp}$ and $d^\pi$, which corresponds to latter two terms in the generalized performance-difference lemma (Lemma~\ref{lemma:PDL}). Now we discuss how to realize both of these two conditions if the function class $\cF$ satisfies certain structural conditions. Two classical methods include BRM-type (e.g., \citet{xie2021bellman}) and MQL-type (e.g., \citet{uehara2020minimax}).

\paragraph{BRM-type critic oracle.}
The BRM-type oracle leverages classical Bellman residual minimization (BRM) algorithm \citep{antos2008learning} with version-space pessimism. For a policy $\pi$ and a value function proxy $f$, define the squared Bellman error (under $d^D_\text{critic}$, which is the state-action distribution of the critic-side data) as
\[
\cE(f;\pi)=\EE_{d^D_\text{critic}}\left[(f-\cT^\pi f)^2\right].
\]
By Bellman residual minimization \citep{antos2008learning}, this can be written as
\[
\cE(f;\pi)=\cL(f;f,\pi)-\cL(\cT^\pi f;f,\pi),
\]
where
\[
\cL(f';f,\pi):=\EE_{(s,a)\sim d^D_\text{critic},r=R(s,a),s'\sim P(\cdot\mid s,a)}\bigl[(f'(s,a)-r-\gamma f(s',\pi))^2\bigr].
\]
Let $\widehat{\cL}$ denote the empirical version of $\cL$ estimated from the offline dataset $\cD$. Under Bellman completeness, i.e., $\cT^\pi f\in\cF$ for all $f\in\cF$, we have that $\cL(\cT^\pi f;f,\pi)=\min_{g\in\cF}\cL(g;f,\pi)$. Hence, we can provide an empirical surrogate of the squared Bellman error as
\[
\widehat{\cE}(f;\pi)=\max_{g\in\cF}\widehat{\cL}(f;f,\pi)-\widehat{\cL}(g;f,\pi).
\]
Therefore, the oracle can be given by the following pessimistic estimation of the Bellman error:
\begin{equation}\label{eq:oracle1}
    f_\tmin^\pi=\underset{f\in\cF_{\epsilon_0}^\pi}{\arg\min}J_f(\pi).
\end{equation}
The pessimism is taken inside the feasible set $\cF_{\epsilon_0}^\pi$, referred to as the \emph{version space}, which contains all function $f\in\cF$ with small empirical squared Bellman error under $d^D$ (Here $\Comp_{\cF,\Pi_\theta}$ denotes for some complexity measure of the policy and value class, e.g., $\log(|\cF||\Pi_\theta|/\delta)$ for finite classes):
\[
\cF_{\epsilon_0}^\pi=\left\{f\in\cF:\widehat{\cE}(f;\pi)\leq \epsilon_0:=\frac{V_\tmax^2}{N}\Comp_{\cF,\Pi_\theta}\right\}.
\]
Since $Q^\pi$ is the unique fixed point of $\cT^\pi$, by standard concentration argument, we have $Q^\pi\in\cF_{\epsilon_0}^\pi$ with probability at least $1-\delta$. Therefore, this oracle (Eq.~\eqref{eq:oracle1}) satisfies the first condition of pessimism ($\epsilon_r=0$) with high probability:
\[
J_f(\pi)=J_{f_\tmin^\pi}(\pi)=\min_{f\in\cF_{\epsilon_0}^\pi} J_f(\pi)\leq J_{Q^\pi}(\pi)=J(\pi).
\]
To see that the oracle (Eq.~\eqref{eq:oracle1}) also satisfies the second argument of bounded Bellman error under $d^{\pi_\cp}$, we need to assume the same coverage condition (Assumption~\ref{ass:data}) also holds for critic-side data, i.e., $\|d^{\pi_\cp}/d^D_\text{critic}\|\le \CC$ for some finite $\CC$. By applying coverage condition (Assumption~\ref{ass:data}), we can transfer the distribution from $d^{\pi_\cp}$ to $d^D_\text{critic}$ (in addition to some concentration argument to relate $\cE(f;\pi)$ with $\widehat{\cE}(f;\pi)$): for all $f\in\cF_{\epsilon_0}^\pi$,
\[
\EE_{d^{\pi_\cp}}[\cT^\pi f-f]\leq\sqrt{\EE_{d^{\pi_\cp}}[(f-\cT^\pi f)^2]}\leq\sqrt{\CC\cdot\EE_{d^D_\text{critic}}[(f-\cT^\pi f)^2]}=\sqrt{\CC\cdot\cE(f;\pi)}\lesssim\sqrt{\CC\epsilon_0}:=\epsilon_b.
\]
Therefore, the BRM-type oracle in Eq.~\eqref{eq:oracle1} is a valid oracle that satisfies both conditions in Assumption~\ref{ass:oracle}. For computational consideration, we write this as a constrained optimization problem:
\[
\min_{f\in\cF}J_f(\pi)+\lambda\widehat{\cE}(f;\pi),
\]
which results in a minimax optimization problem due to the $\max_g$ in the estimation of $\widehat{\cE}(f;\pi)$. When $\cF$ is a linear function class with respect to some feature $\{\phi_{s,a}\}$, this optimization problem reduces to a quadratic program which can be efficiently solved, as discussed in \citep{xie2021bellman}.

\paragraph{MQL-type critic oracle.}
Another implementation of the critic oracle is based on minimax Q-function learning (MQL), a marginalized importance sampling method. It only requires the realizability of function class $\cF$ (i.e., $Q^\pi\in\cF$ for all $\pi\in\Pi_\theta$) instead of Bellman completeness. The oracle is still given by Eq.~\eqref{eq:oracle1} but with a different version space $\cF_{\epsilon_0}^\pi$ defined as
\[
\cF_{\epsilon_0}^\pi=\left\{f\in\cF:\max_{w\in\cW}\widehat{\ell}(f;w,\pi)\leq\epsilon_0':=\CC V_\tmax\sqrt{\frac{\Comp_{\cF,\Pi_\theta,\cW}}{N}}\right\},
\]
where $\cW$ is a weight class similar to what is defined in Section~\ref{sec:dro}. And $\widehat{\ell}(f;w,\pi)$ is the empirical version of $\ell(f;w,\pi)$, the average Bellman error of $f$ under weighted distribution $w\cdot d^D_\text{critic}$:
\[
\ell(f;w,\pi)=\Bigl|\EE_{d^D_\text{critic}}[w\cdot(f-\cT^\pi f)]\Bigr|.
\]
Since $Q^\pi$ is the unique fixed point of $\cT^\pi$, $\ell(Q^\pi;w,\pi)=0$ for all $w\in\cW$. By concentration argument, we have $Q^\pi\in\cF_{\epsilon_0}^\pi$ with probability at least $1-\delta$. This indicates the pessimism is satisfied when we are taking $\min_{f\in\cF_{\epsilon_0}^\pi}$ in Eq.~\eqref{eq:oracle1}, validating the first condition ($\epsilon_r=0$). Similarly, the second condition of bounded Bellman error under $d^{\pi_\cp}$ can be analyzed since $\|d^{\pi_\cp}/d^D_\text{critic}\|_\infty\leq C$ by coverage condition (Assumption~\ref{ass:data}): for all $f\in\cF_{\epsilon_0}^\pi$,
\[
\EE_{d^{\pi_\cp}}[\cT^\pi f-f]\leq \Bigl|\EE_{d^D_\text{critic}}[w^*(f-\cT^\pi f)]\Bigr|\leq\max_{w\in\cW}\Bigl|\EE_{d^D_\text{critic}}[w(f-\cT^\pi f)]\Bigr|=\max_{w\in\cW}\ell(f;w)\lesssim \epsilon_0':=\epsilon_b,
\]
where $w^*(s,a)=d^{\pi_\cp}(s,a)/d^D_\text{critic}(s,a)$ and $w^*\in\cW$ (realizability of the weight class). Therefore, this MQL-type oracle in Eq.~\eqref{eq:oracle1} also satisfies both conditions in Assumption~\ref{ass:oracle}. For computational consideration, it requires an additional weight class $\cW$. And this can also be treated as a minimax optimization problem:
\[
\min_{f\in\cF}\max_{w\in\cW} J_f(\pi)+\lambda\widehat{\ell}(f;w,\pi).
\]
In special cases like linear function approximation, it can be transformed to some mean-matching problem which can be efficiently solved.

\subsection{Settings in Continuous Action Space}\label{app:continuous-setting}

For readability, the main text states Section~\ref{sec:pspi} under a finite action space and writes action-wise expectations using sums. In this appendix, we record the corresponding continuous-action formulation of the state-wise mirror-descent argument underlying Theorem~\ref{thm:pspi}. This appendix only extends the PSPI analysis in Section~\ref{sec:pspi}, the rest of the paper remains stated in the finite-action setting.

\paragraph{Setting.}

We keep the state space unchanged (since the contextual analysis does not depend on the size of state space) and only generalize the action space. Let $\cA$ be a measurable action space equipped with a $\sigma$-finite reference measure $\nu$. A policy $\pi$ is a Markov kernel from states to actions such that, for every state $s$, the conditional action law $\pi(\cdot\mid s)$ is absolutely continuous with respect to $\nu$. We write its Radon-Nikodym derivative as $\pi(a\mid s)$, so that $\int_\cA\pi(a\mid s)\nu(\mathrm{d}a)=1$, $\forall s\in\cS$. Accordingly, whenever a function $g:\cS\times\cA\to\RR$ is measurable, we write $g(s,\pi):=\int_\cA g(s,a)\pi(a\mid s)\nu(\mathrm{d}a)$. With this notation, the definition of Bellman operator $(\cT^\pi f)(s,a)=R(s,a)+\gamma\EE_{s'\sim P(\cdot\mid s,a)}[f(s',\pi)]$ immediately extends to continuous action spaces as-is.

To avoid ambiguity, it is helpful to distinguish the discounted state occupancy from the discounted state-action occupancy. Denote $d_\cS^\pi$ the discounted state occupancy measure of $\pi$: $d_\cS^\pi(s)=(1-\gamma)\sum_{t=0}^\infty\gamma^t\Pr_\pi[s_t=s]$ (recall that we still treat states as finite and discrete), and denote by $d^\pi$ the discounted state-action occupancy measure on $\cS\times\cA$: $d^\pi(B)=(1-\gamma)\sum_{t=0}^\infty\gamma^t\Pr_\pi[(s_t,a_t)\in B]$ 
for measurable sets $B\subseteq\cS\times\cA$. In particular, the regret quantity in Eq.~\eqref{eq:regret} should be interpreted as an expectation over the state marginal $d_\cS^{\pi_{\cp}}$, namely
\[
\frac{\Reg_K}{K}=\frac1K\sum_{k=1}^K\EE_{s\sim d_\cS^{\pi_\cp}}[f_k(s,\pi_\cp)-f_k(s,\pi_k)].
\]
This is the same quantity as in the main text, just written with the state occupancy explicitly separated from the state-action occupancy. 

For any two action distributions $\mu$ and $\lambda$ on $\cA$ such that $\mu$ absolutely continuous with respect to $\lambda$, we define their KL-divergence as $\KL(\mu\|\lambda)=\int_\cA\log(\mathrm{d}\mu/\mathrm{d}\lambda)\mathrm{d}\mu$. Furthermore, if both $\mu$ and $\lambda$ admit densities $p$ and $q$, respectively, with respect to $\nu$, the KL reduces to $\KL(p\|q)=\int_\cA p(a)\log(p(a)/q(a))\nu(\mathrm{d}a)$. Similar to the main text, we define the expected KL divergence under $d_\cS^{\pi_\cp}$ between policy $\pi_\cp$ and $\pi_1$ as: 
\[
\KL(\pi_\cp\|\pi_1)=\EE_{s\sim d_\cS^{\pi_\cp}}[\KL(\pi_\cp(\cdot\mid s)\|\pi_1(\cdot\mid s))].
\]

\paragraph{Algorithm and guarantee.}

Under this setup, the PSPI update has the density form
\begin{equation}\label{eq:pspi-continuous}
    \pi_{k+1}(a\mid s)=\frac{\pi_k(a\mid s)\exp(\eta f_k(s,a))}{\int_\cA \pi_k(a'\mid s)\exp(\eta f_k(s,a'))\nu(\mathrm{d}a')}.
\end{equation}
This is exactly the continuous analogue of the state-wise multiplicative weights update used in PSPI. The reason this extension is clean is that Section~\ref{sec:pspi} is purely state-wise: once one fixes a state $s$, the proof is just an online learning argument over the action distribution at that state. 

We now state the continuous-action extension of Theorem~\ref{thm:pspi}, which exhibits the same structure as that in finite action space. In fact, it strictly subsumes Theorem~\ref{thm:pspi}, since the latter can be instantiated by simply choosing the base measure $\nu$ as the counting measure over a finite $\cA$. Therefore, below we only prove Theorem~\ref{thm:continuous-pspi}, as Theorem~\ref{thm:pspi} is a direct corollary of it.\footnote{To see that, by choosing the base measure $\nu$ as the counting measure in $\cA$. Then, all continuous notions defined in this section reduce back to the finite and discrete notions in the main text.}

\begin{theorem}[Theorem~\ref{thm:pspi} with continuous action space]\label{thm:continuous-pspi}
    Assume that for each iteration $k$, the critic $f_k$ is bounded in $[0,V_\tmax]$, and that $\KL(\pi_\cp\|\pi_1)<\infty$. Then the PSPI iterates satisfy
    \[
    \frac{\Reg_K}{K}\le\frac{\KL(\pi_\cp\|\pi_1)}{\eta K}+\frac{\eta V_\tmax^2}{8}.
    \]
    By choosing the step size as $\eta=\sqrt{8\KL(\pi_\cp\|\pi_1)/(KV_\tmax^2)}$, we have exactly the same regret bound as in Theorem~\ref{thm:pspi} (except here the KL divergence is defined in a continuous manner), i.e., $\Reg_K/K\le V_\tmax\sqrt{\KL(\pi_\cp\|\pi_1)/(2K)}$.
\end{theorem}

We now give a potential-based proof that utilizes classical online learning tools in the analysis of expert problem \citep{hazan2016introduction, orabona2019modern}.

\begin{proof}[Proof of PSPI regret bound]
Let $\xi:L^\infty(\cA)\to\RR$ be the negative potential function defined as
\[
\xi(f)=-\frac{1}{\eta}\log\int_\cA \pi_1(a\mid s)\exp(\eta f(a))\nu(\mathrm{d}a).  
\]
Let $F_k(s,a)=\sum_{j=1}^k f_j(s,a)$ denote the cumulative value function (at each fixed state $s$) up to round $k$. 
We can first bound the difference of the negative potential function as
\begin{align*}
\xi(F_{k+1}(s,\cdot))-\xi(F_k(s,\cdot))
&=-\frac{1}{\eta}\log\frac{\int_\cA  \pi_1(a\mid s)\exp(\eta\sum_{j=1}^{k+1}f_k(s,a))\nu(\mathrm{d}a)}{\int_\cA  \pi_1(a\mid s)\exp(\eta\sum_{j=1}^{k}f_k(s,a))\nu(\mathrm{d}a)}\\
&=-\frac{1}{\eta}\log\int_\cA  \pi_{k+1}(a\mid s)\exp(\eta f_{k+1}(s,a))\nu(\mathrm{d}a)\\
&=-\frac{1}{\eta}\log\mathbb{E}_{a\sim \pi_{k+1}(\cdot\mid s)}[\exp(\eta f_{k+1}(s,a))]\\
&\geq-\mathbb{E}_{a\sim \pi_{k+1}(\cdot\mid s)}[f_{k+1}(s,a)]-\frac{\eta V_\tmax^2}{8},
\end{align*}
where the last step we use the Hoeffding's lemma with random variable $f_{k+1}(s,a)$ (the randomness comes from $a\sim \pi_{k+1}(\cdot\mid s)$) and the range of $f_{k+1}(s,a)$ is $[0,V_\tmax]$. Therefore, by telescoping, we have
\[
\sum_{k=1}^K\left(\xi(F_{k}(s,\cdot))-\xi(F_{k-1}(s,\cdot))\right)\geq-\sum_{k=1}^K\mathbb{E}_{a\sim \pi_k(\cdot\mid s)}[f_k(s,a)]-\frac{\eta V_\tmax^2}{8}K.
\]
Rearranging the terms we have
\[
-\sum_{k=1}^K\mathbb{E}_{a\sim \pi_k(\cdot\mid s)}[f_k(s,a)]\leq\frac{\eta V_\tmax^2}{8}K+\xi(F_K(s,\cdot)),
\]
since $\xi(F_0(s,\cdot))=\xi(0)=-\frac{1}{\eta}\log\int_\cA \pi_1(a\mid s)\nu(\mathrm{d}a)=-\frac{1}{\eta}\log 1=0$ (by setting $\pi_0\equiv \pi_1$ as the initial policy). By adding $F_K(s,\pi_\cp)=\sum_{k=1}^Kf_k(s,\pi_\cp)$ at both sides, we can therefore obtain the regret bound
\[
\sum_{k=1}^K\left(f_k(s,\pi_\cp)-f_k(s,\pi_k)\right)\leq F_K(s,\pi_\cp)+\frac{\eta V_\tmax^2}{8}K+\xi(F_K(s,\cdot)).
\]
This means we only need to bound the negative log-partition function at round $K$, i.e., $\xi(F_k(s,\cdot))$. This follows directly via the Gibbs variational principle (Lemma~\ref{lemma:gibbs}) with base measure $\nu$:
\[
\xi(F_K(s,\cdot)) = \inf_{u\in\Delta_\nu(\cA)}\Big\{\frac{1}{\eta}\KL(u\|\pi_1(\cdot\mid s)) - \EE_{a\sim u}[F_K(s,a)]\Big\}.
\]
Taking $u=\pi_\cp(\cdot\mid s)$ and using the fact that $\EE_{a\sim \pi_\cp(\cdot\mid s)}[F_K(s,a)]=\sum_{k=1}^K \EE_{a\sim \pi_\cp(\cdot\mid s)}[f_k(s,a)]$, we obtain
\[
\sum_{k=1}^K\left(f_k(s,\pi_\cp)-f_k(s,\pi_k)\right)\le \frac{\eta V_\tmax^2}{8}K + \frac{1}{\eta}\KL(\pi_\cp(\cdot\mid s)\|\pi_1(\cdot\mid s))]).
\]
By taking the outer expectation over the designated state distribution $s\sim d_\cS^{\pi_\cp}$ and tuning the step size as $\eta=\sqrt{(8\KL(\pi_\cp\|\pi_1))/(KV_\tmax^2)}$, we achieve the regret bound in Theorem~\ref{thm:pspi}.
\end{proof}

\subsection{Examples of Bounded KL Term in Theorem~\ref{thm:continuous-pspi}}\label{app:continuous-ex}

We deliberately keep the KL term in Theorem~\ref{thm:pspi} and Theorem~\ref{thm:continuous-pspi}, rather than replacing it by a crude worse-case upper bound. In the discrete setting, one often upper-bounds this term by $\log|\cA|$, but such a reduction is neither natural nor informative in general continuous action spaces. More importantly, the KL term itself admits elegant and interpretable instantiations for many standard continuous policy classes. Below we provide two examples to further illustrate this point.

\paragraph{Gaussian policies.}
We first consider Gaussian policies, under which the KL term $\KL(\pi_\cp\|\pi_1)$ admits explicit closed-form expressions and can be further bounded under mild regularity conditions. For the sake of simplicity, we focus on the isotropic Gaussian case, which is widely used in continuous-control reinforcement learning due to its simplicity, numerical stability, and compatibility with policy-gradient methods. In particular, isotropic Gaussian policies arise naturally when actions are modeled as affine functions of the state with additive exploration noise, and they are commonly adopted in practical algorithms (e.g., actor-critic methods with fixed variance) \citep{haarnoja2018soft}. Concretely, we assume that for each state $s$,
\[
\pi_i(\cdot\mid s)=\mathcal N(\mu_i(s),\sigma^2 I),\qquad i\in\{1,\cp\}.
\]
Then, for any fixed $s$, the KL divergence reduces to a quadratic form in the mean difference:
\begin{align*}
\KL\bigl(&\pi_\cp(\cdot\mid s)\|\pi_1(\cdot\mid s)\bigr)
= \EE_{a\sim \pi_\cp(\cdot\mid s)}\left[\log\frac{\pi_\cp(a\mid s)}{\pi_1(a\mid s)}\right] \\
&= \EE_{a}\left[-\frac{1}{2}(a-\mu_\cp(s))^\top(\sigma^2 I)^{-1}(a-\mu_\cp(s))
+\frac12(a-\mu_1(s))^\top(\sigma^2 I)^{-1}(a-\mu_1(s))\right] \\
&= \frac{1}{2\sigma^2}\EE_{a}\left[\|a-\mu_1(s)\|_2^2-\|a-\mu_\cp(s)\|_2^2\right] \\
&= \frac{1}{2\sigma^2}\|\mu_\cp(s)-\mu_1(s)\|_2^2.
\end{align*}
Therefore,
\[
\KL(\pi_\cp\|\pi_1)
=\EE_{s\sim d^{\pi_\cp}}\left[\KL\left(\pi_\cp(\cdot\mid s)\middle\|\pi_1(\cdot\mid s)\right)\right]
=\frac{1}{2\sigma^2}\EE_{s\sim d^{\pi_\cp}}\left[\|\mu_\cp(s)-\mu_1(s)\|_2^2\right].
\]
Moreover, if the mean difference is uniformly bounded, i.e., $\|\mu_\cp(s)-\mu_1(s)\|_2\le B$ for all $s$, then
\[
\KL(\pi_\cp\|\pi_1)\le \frac{B^2}{2\sigma^2}.
\]
One important advantage of Gaussian policies is that they admit \textbf{efficient computation}. That is, the update does not involve integrating the whole (continuous) action space, but only updates the sufficient statistics.\footnote{This is not restricted to Gaussian policies. In general, under certain regularity conditions, for policies that can be expressed as a distribution in exponential family, we can all update them via sufficient statistics.} To be specific, we can consider the function class $\cF$ as a class of quadratic functions (which is typical in settings like linear quadratic regulators, LQR). Suppose the pessimistic $Q$-function at round $k$ can be written as (with respect to $a$)
\[
f_k(s,a)=-\frac{1}{2}(a-u_k)^\top Q_k (a-u_k) + c_k,
\]
with state-dependent parameters $u_k, Q_k, c_k$ with $Q_k\succ0$. In this case the multiplicative-weight update in Eq.~\eqref{eq:md} preserves the Gaussian family: the density $\pi_k(a\mid s)$ remains a multivariate Gaussian $\cN(\mu_k,\Sigma_k)$ after the update. Its sufficient statistics are the mean $\mu_k$ and covariance $\Sigma_k$, which can be equivalently represented by the precision matrix $\Lambda_k = \Sigma_k^{-1}$ and the natural parameter $h_k=\Sigma_k^{-1}\mu_k$. The Hedge update then reduces to
\[
\Lambda_{k+1} = \Lambda_k + \eta Q_k, \quad h_{k+1} = h_k + \eta Q_k u_k,
\]
so that the updated Gaussian is given by $\cN(\mu_{k+1}, \Sigma_{k+1})$ with $\mu_{k+1} = \Lambda_{k+1}^{-1}h_{k+1}$ and $\Sigma_{k+1}=\Lambda_{k+1}^{-1}$. 

\paragraph{Convex action space.}
As claimed in Section~\ref{sec:pspi}, under additional structural assumptions on the action space (i.e., convexity) and the function class (i.e., Lipschitzness), we can obtain unified bound that does \emph{not} depend on the KL divergence of any specific policies. The key idea is to leverage the convexity of the action space $\cA$, which allows us to construct a sequence of subsets $\cA_k\subset\cA$ that progressively approximate the deterministic action chosen by $\pi_\cp$. Such a structure can also be extended to notions of \emph{uniform-fatness}, where every action is guaranteed to have sufficient volume in its neighborhood (see \citet{krichene2015hedge} for more discussion).

\begin{theorem}[Unified KL Bound with Convex Action Space for Theorem~\ref{thm:continuous-pspi}]\label{thm:pspi-unified}
Suppose the action space $\cA$ has finite measure with respect to the Lebesgue measure $\nu$, i.e., $\nu(\cA)<\infty$. Furthermore, assume that the action space $\cA\subset\RR^{\dd_\cA}$ is convex and compact with diameter $B$ in the $\dd_\cA$-dimensional Euclidean space $\RR^{\dd_\cA}$, and that $f_k(s,\cdot)$ is $L$-Lipschitz for all $k$, then the PSPI update with initial distribution $\pi_1(\cdot\mid s)$ being Lebesgue uniform for all $s\in\cS$, i.e., $\pi_1(a\mid s)\equiv 1/\nu(\cA)$ for all $a\in\cA$, and step size $\eta=\sqrt{(8\dd_\cA\log K)/(KV_\tmax^2)}$ achieves
\[
\frac{\Reg_K}{K}\leq V_\tmax\sqrt{\frac{\dd_\cA\log K}{2K}}+\frac{LB}{K}.
\]
\end{theorem}

Note that the additional $\dd_\cA\log K$ factor in Theorem~\ref{thm:pspi-unified} (compared to regular $\sqrt{\log|\cA|/K}$ regret for Hedge algorithm) can be viewed as the $\log|\cA|$ term in the discrete case: the algorithm is effectively \emph{learning from a finite cover} of $\cA$. For a $\dd_\cA$-dimensional set $S$, its log-covering number (see Definition~\ref{def:covering-number} for formal definition) scales as $\log\cN(S,\epsilon)\asymp \dd_\cA\log(1/\epsilon)$; choosing $\epsilon=1/\sqrt{K}$ (which is the rate of the no-regret algorithm) yields the $\dd_\cA\log K$ term. This is why the dimension $\dd_\cA$ can be viewed as some \textup{effective dimension} of $\cA$, as actions in $\cA$ can be embedded as the $\dd_\cA$-dimensional vector in $\RR^{\dd_\cA}$.

\begin{proof}[Proof of Theorem~\ref{thm:pspi-unified}]
Fix an arbitrary state $s$. Let $F_K(s,a)=\sum_{k=1}^K f_k(s,a)$ denote the cumulative value function. Since each $f_k(s,\cdot)$ is Lipschitz and $\cA$ is compact, $F_K(s,\cdot)$ attains its maximum on $\cA$. Let $a^*\in\arg\max_{a\in\cA}F_K(s,a)$. Then, for any comparator policy $\pi_\cp(\cdot\mid s)$, possibly deterministic, we have
\begin{align*}
\sum_{k=1}^K\left(f_k(s,\pi_\cp)-f_k(s,\pi_k)\right)
&=F_K(s,\pi_\cp)-\sum_{k=1}^Kf_k(s,\pi_k)\\
&\le F_K(s,a^*)-\sum_{k=1}^Kf_k(s,\pi_k),
\end{align*}
where the inequality follows because the expectation of $F_K(s,\cdot)$ under any probability measure is upper bounded by its maximum over $\cA$.

The convexity of the action space makes it possible to construct a sequence of action subsets $\cA_k\subset \cA$ such that
\[
\cA_k=\left\{a^*+d_k(a-a^*):a\in\cA\right\},
\]
where $d_k> 0$ is called the \emph{decay rate} (will be specified later). This is because the diameter of $\cA_k$ will decay according to $d_k$, i.e., $\sup_{a,a'\in \cA_k}\|a-a'\|=d_k\sup_{a,a'\in\cA }\|a-a'\|=d_k B$, where $B$ is the diameter of the action space $\cA $. 

Notice that all of the $\cA_k$ satisfy \emph{realizability}, i.e., $a^*\in \cA_k$ for all $k\in[K]$. Specifically we consider $\cA_k$, the action subset at the last round. By Lipschitzness of $f_k(s,\cdot)$, for all $a\in \cA_k$, we have 
\[
|f_k(s,a)-f_k(s,a^*)|\leq L\|a^*-a\|=d_KLB.
\]
This means $f_k(s,a)\geq f_k(s,a^*)-d_KLB$. Summing over $k\in[K]$, we have the cumulative bound
\[
F_K(s,a)=\sum_{k=1}^Kf_k(s,a)\geq\sum_{k=1}^Kf_k(s,a^*)-Kd_KLB=F_K(s,a^*)-Kd_KLB.
\]
Recall that in the proof of Theorem~\ref{thm:pspi}, we need to control the last-iterate potential function $\xi(F_K(s,\cdot))$:
\begin{align*}
\xi(F_K^s)&=-\frac{1}{\eta}\log\int_\cA \pi_1(a\mid s)\exp(\eta F_K^s(a))\nu(\mathrm{d}a)\\
&\leq -\frac{1}{\eta}\log\int_{\cA_k}\pi_1(a\mid s)\exp(\eta F_K(s,a))\nu(\mathrm{d}a)\\
&\leq-\frac{1}{\eta}\log\int_{\cA_k}\pi_1(a\mid s)\exp(\eta (F_K(s,a^*)-Kd_KLB))\nu(\mathrm{d}a)\\
&=\left(Kd_KLB-F_K(s,a^*)\right)-\frac{1}{\eta}\log\int_{\cA_k}\pi_1(a\mid s)\nu(\mathrm{d}a).
\end{align*}
Therefore, we can get the regret
\[
\Reg_K^s\leq F_K(s,a^*)+\frac{\eta V_\tmax^2}{8}K+\xi(F_K(s,\cdot))\leq \frac{\eta V_\tmax^2}{8}K+Kd_KLB-\frac{1}{\eta}\log\int_{\cA_k}\pi_1(a\mid s)\nu(\mathrm{d}a).
\]
We only need to control the last term, which is actually the \emph{generalized volume} of $\cA_k$. For a set $\cA_0\subset\cA $, its generalized volume with respect to distribution $\pi_1(\cdot\mid s)\in\Delta_\nu(\cA )$ is defined as
\[
\mathcal{V}_{\pi_1(\cdot\mid s)}(\cA_0)=\int_{\cA_0}\pi_1(a\mid s)\nu(\mathrm{d}a)=\frac{1}{\nu(\cA )}\int_{\cA_k}\nu(\mathrm{d}a)=\frac{\nu(\cA_k)}{\nu(\cA )}
\]
where the last equality is due to $\pi_1(a\mid s)=1/\nu(\cA)$ is the Lebesgue uniform measure on $\cA$. According to our construction of $\cA_k$, using change of variables,
\begin{align*}
\nu(\cA_k)&=\int_{\cA_k}\nu(\mathrm{d}a)=\int_\cA \mathrm{d}\nu(a^*+d_K(a-a^*))=\int_\cA |\det(d_KI_{\dd_\cA})|\nu(\mathrm{d}a)\\&=(d_K)^{\dd_\cA}\int _A\nu(\mathrm{d}a)=(d_K)^{\dd_\cA}\nu(\cA ).
\end{align*}
This means the generalized volume of $\cA_k$ is $(d_K)^{\dd_\cA}$. Hence, by setting the decay rate $d_K=1/K$, we have
\[
\sum_{k=1}^K\left(f_k(s,\pi_\cp)-f_k(s,\pi_k)\right)\leq\frac{\eta V_\tmax^2}{8}K+LB+\frac{\dd_\cA}{\eta}\log K.
\]
Taking the expectation over the designated state distribution $s\sim d_\cS^{\pi_\cp}$ with optimally tuned step size $\eta$ to be $\eta=\sqrt{8\dd_\cA\log K/(KV_\tmax^2)}$. This leads to the argument in Theorem~\ref{thm:pspi-unified}.
\end{proof}

\section{Omitted Details for Section~\ref{sec:challenge}}\label{app:challenge}

In this section, we first provide positive and negative examples of \dataups in the sense of Definition~\ref{def:local-update}.
We then present the construction and proof of the hardness result in Proposition~\ref{thm:hardness}, followed by further discussion of the underlying contextual coupling issue with additional symmetry condition and data admissibility.

\subsection{Examples of (Non-)\DataUps}\label{app:challenge-alg}

\paragraph{Contextual mirror descent.}
Recall the state-wise mirror descent update rule in Eq.~\eqref{eq:md}. Its multiplicative weight update $\pi_{k+1}(a\mid s)\propto\pi_k(a\mid s)\exp(\eta f_k(s,a))$ can be expressed as
\[
\pi_{k+1}(\cdot\mid s)=\underset{\pi(\cdot\mid s)\in\Delta(\cA)}{\arg\max} \left\{f_k(s,\pi)-\tfrac{1}{\eta}\KL(\pi(\cdot\mid s)\|\pi_k(\cdot\mid s))\right\}.
\]
Hence, a natural idea is to \emph{contextualize} the above optimization problem through the ``local'' distribution $s\sim d^D$, as we want guarantees under $d^D$ to be transferred to guarantees under $d^{\pi_\cp}$. We call this natural algorithm \textbf{\emph{contextual mirror descent}} (also known as \emph{policy mirror descent} \citep{lan2023policy,zhan2023policy,sherman2025convergence}), whose update rule is given by
\begin{equation}\label{eq:cmd}
    \pi_{k+1}=\underset{\pi\in\Pi_\theta}{\arg\max}\ \EE_{s\sim d^D}\Bigl[f_k(s,\pi)-\tfrac1\eta\KL(\pi(\cdot\mid s)\|\pi_k(\cdot\mid s))\Bigr].
\end{equation}
Equivalently, contextual mirror descent can be viewed as performing state-wise mirror descent for each state $s\in\cS$ at first, leading to a policy $\widetilde{\pi}_{k+1}$; and then \emph{project} back $\widetilde{\pi}_{k+1}$ to $\Pi_\theta$ to get $\pi_{k+1}$, in the sense of expected KL divergence, i.e., $\pi_{k+1}=\arg\min_{\pi\in\Pi_\theta}\EE_{s\sim d^D}[\KL(\pi(\cdot\mid s)\|\widetilde{\pi}_{k+1}(\cdot\mid s))]$.

\paragraph{Local updates.}

We first provide an equivalent characterization of Definition~\ref{def:local-update}. For any measurable test function $\psi$, define the population expectation operator
\[
\cE_k[\psi]:=\EE_{s\sim d^D}\Bigl[\psi\bigl(d^D(\cdot\mid s),\theta\mapsto\pi_\theta(\cdot\mid s),\{f_j(s,\cdot)\}_{j\le k}\bigr)\Bigr].
\]
Since the joint distribution $\PP_k$ in Definition~\ref{def:local-update} is uniquely characterized by its expectations against \emph{any} measurable test functions $\psi$, specifying how an algorithm relies on $\PP_k$ boils down to explicitly writing down $\cE_k$. Thus, Definition~\ref{def:local-update} can be equivalently expressed as $v_k\sim \UU_k(\cdot\mid \theta_{1:k},\cE_k)$.

The next proposition shows that contextual mirror descent, along with the LSPU and DRPU developed in Section~\ref{sec:algorithm}, belong to the class of local updates in Definition~\ref{def:local-update}.

\begin{proposition}[Contextual Mirror Descent, LSPU, and DRPU are Local Updates]\label{prop:def-update}
    Consider the following population actor-update methods: 1) contextual mirror descent in Eq.~\eqref{eq:cmd}; 2) LSPU in Eq.~\eqref{eq:ols-solution}; 3) DRPU in Eq.~\eqref{eq:cvar-lp} with bounded-density-ratio class $\cW_\infty$. All three methods satisfy Definition~\ref{def:local-update}.
\end{proposition}
\begin{proof}
    We show that the objectives of all three methods can be expressed through the expectation operator $\cE_k$. For convenience, we write the arguments of a test function $\psi$ as $(b,\Phi,f_{1:k})$ where $b=d^D(\cdot\mid s)$ is the conditional action distribution in data, $\Phi:\theta\mapsto\pi_\theta(\cdot\mid s)\in\Delta(\cA)$ is the policy parameterization, and $f_j$ is an action-value vector (with a little abuse of notation).

    1) \emph{Contextual mirror descent}. For any candidate update direction $v$, let the test function be
    \[
    \psi_{k,v}^\mathrm{CMD}(b,\Phi,f_{1:k})=\sum_a\Phi(\theta_k+\eta v)(a)f_k(a)-\frac1\eta\KL(\Phi(\theta_k+\eta v)\|\Phi(\theta_k)).
    \]
    Then the population objective in Eq.~\eqref{eq:cmd} is exactly $F_k(v)=\cE_k[\psi_{k,v}^\mathrm{CMD}]$. Therefore, its optimizer can be determined entirely by $(\theta_{1:k},\cE_k)$.

    2) \emph{LSPU}. Define $A_k^\Phi(a):=f_k(a)-\sum_{a'}\Phi(\theta_k)(a')f_k(a')$ and $g_k^\Phi(a):=\nabla_\theta\log\Phi(\theta_k)(a')$, both of which are determined by the arguments of $(b,\Phi,f_{1:k})$, as well as the current parameter $\theta_k$. Let the test function be
    \[
    \psi_{k,b}^\mathrm{LS}(b,\Phi,f_{1:k})=\sum_ab(a)(A_k^\Phi(a)-v^\top g_k^\Phi(a))^2.
    \]
    Then the population least-square objective satisfies $L_k(v)=\cE_k[\psi_{k,v}^\mathrm{LS}]$, so its constrained minimizer also depends only on $(\theta_{1:k},\cE_k)$.

    3) \emph{DRPU}. Similarly, for any sign $\sign\in\{\pm 1\}$ and threshold $\tau\in\RR$, let the test function be
    \[
    \psi_{k,v,\sign,\tau}^\mathrm{DR}(b,\Phi,f_{1:k})=\sum_ab(a)\bigl(\sign (A_k^\Phi(a)-v^\top g_k^\Phi(a))-\tau\bigr)_+.
    \]
    By the CVaR dual representation (Proposition~\ref{prop:cvar}), the population DRPU objective can be written as $\ell_k(v)=\max_{\sign\in\{\pm1\}}\inf_{\tau\in\RR}\left\{\tau+ \CC\cdot \cE_k[\psi_{k,v,\sign,\tau}^\mathrm{DR}]\right\}$. Thus, the entire robust objective is determined by $(\theta_{1:k},\cE_k)$, without querying any individual state identities or additional MDP information.
\end{proof}

\paragraph{``Non-local'' updates.}

At the end of Section~\ref{sec:challenge}, we mentioned that there are indeed some algorithms that can violate the ``local'' update rule in Definition~\ref{def:local-update}, such as those that leverage ``non-local'' on-policy information. Without loss of generality, we assume there is some behavior policy $\pi_b$ such that $d^D=d^{\pi_b}$, the discounted occupancy distribution of $\pi_b$. Recall that $A_k(s,a)=f_k(s,a)-f_k(s,\pi_k)$. Consider the following two (population) actor updates:
\begin{enumerate}[leftmargin=*]
    \item \emph{Trajectory-wise importance-sampling policy gradient} (IS-PG):
    \begin{equation}\label{eq:is-pg}
    v_k^\mathrm{IS}=(1-\gamma)\EE_{\tau\sim\pi_b}\left[\sum_{t=0}^\infty\gamma^t\rho_{0:t}A_k(s_t,a_t)\nabla_\theta\log\pi_k(a_t\mid s_t)\right],
    \end{equation}
    where $\rho_{0:t}=\prod_{i=0}^t\pi_k(a_i\mid s_i)/\pi_b(a_i\mid s_i)$ is the trajectory-wise IS ratio.

    \item \emph{Marginalized importance-sampling natural policy gradient} (MIS-NPG):
    \begin{equation}\label{eq:mis-npg}
        v_k^\mathrm{MIS}\in\underset{v:\|v\|\le B_L}{\arg\min}\ \EE_{d^D}\Bigl[w_k(s,a)\bigl(A_k(s,a)-v^\top\nabla_\theta\log\pi_k(a\mid s)\bigr)^2\Bigr],
    \end{equation}
    where $w_k(s,a)=d^{\pi_k}(s,a)/d^D(s,a)$ is the MIS ratio.
\end{enumerate}

The following propositions shows these two are ``negative examples'' against Definition~\ref{def:local-update}.

\begin{proposition}\label{prop:def-non-update} 
Consider  IS-PG in Eq.~\eqref{eq:is-pg} and MIS-NPG in Eq.~\eqref{eq:mis-npg}. Both violate Definition~\ref{def:local-update} and neither is a local update rule.
\end{proposition}

\begin{proof}
    By the principle of importance sampling, the IS-PG update can be written as
    \[
    v_k^\mathrm{IS}=\EE_{(s,a)\sim d^{\pi_k}}[A_k(s,a)\nabla_\theta\log\pi_k(a\mid s)],\\
    \]
    where the expectation is evaluated through trajectory-wise importance ratios. Similarly, MIS-NPG uses the correction weight $w_k=d^{\pi_k}/d^D$, yielding the update
    \[
    v_k^\mathrm{MIS}\in\underset{v:\|v\|\le B_L}{\arg\min}\ \EE_{d^{\pi_k}}[(A_k(s,a)-v^\top\nabla_\theta\log\pi_k(a\mid s))^2].
    \]
    Unlike the local updates in Proposition~\ref{prop:def-update}, both methods depend on the target-policy occupancy $d^{\pi_k}$, which is determined by the MDP's transition dynamics and is not necessarily identifiable from the contextual information law $\PP_k$ (or equivalently, the operator $\cE_k$).

    To see this, consider two MDPs $\cM_1, \cM_2$ with identical reward functions, policy classes, critics, and offline data distributions, but different transition kernels. Specifically, let $\cS=\{s_1,s_2\}$, $\cA=\{0,1\}$, $\gamma=1/2$, $d_0=\delta_{s_1}$, and $\pi_b(1\mid s)=1/2$. In the first MDP, the next state is uniformly distributed regardless of the action; in the second, action $0$ leads to $s_1$ and action $1$ leads to $s_2$. Both MDPs induce the same $d^D=d^{\pi_b}$.

    Choose the policy class $\pi_\theta(1\mid s_1)=\pi_\theta(0\mid s_2)=\sigma(\theta)$, initialized at $\theta_1=\log 3$, and set $R(s,a)=1+a-\pi_1(1\mid s)$. One can verify that the two MDPs have identical exact critics $f_1=Q^{\pi_1}$, and therefore identical contextual information law $\PP_1$. However, their target-policy distribution differ:
    \[
    d_{\cM_1}^{\pi_1}=\left(\frac34,\frac14\right),\quad d_{\cM_2}^{\pi_2}=\left(\frac7{10},\frac3{10}\right).
    \]
    Consequently, direct calculation gives
    \[
    v_1^\mathrm{IS}(\cM_1)=\frac3{32},\quad v_1^\mathrm{IS}(\cM_2)=\frac{3}{40};\quad v_1^\mathrm{MIS}(\cM_1)=\frac{1}{2},\quad v_1^\mathrm{MIS}(\cM_2)=\frac25,
    \]
    where the MIS constraint is taken with $B_L\ge 1$. Thus, despite having identical parameter histories and statistics on the data distribution, the two MDPs yield different update directions. Neither methods can therefore be represented solely as a function of $(\theta_{1:k},\PP_{1:k})$, violating Definition~\ref{def:local-update}.
\end{proof}

\subsection{Proof of Proposition~\ref{thm:hardness}}
\label{app:challenge-proof}

\paragraph{Contextual bandits construction.}

Consider two contextual bandits instances $\cM_+$ and $\cM_-$,
with $\cS=\{s_1,s_2\}$, $\cA=\{0,1\}$ and $\gamma=0$. Both instances have the same reward function $R(s,a)=a$ and uniform offline distribution $d^D(s,a)=1/4$ for all $(s,a)$. Consider the one-dimensional log-linear policy class $\pi_\theta(1\mid s)=\sigma(\theta x(s))$, where $x(s_1)=1$, $x(s_2)=-1$, and $\sigma(z)=1/(1+e^{-z})$. Thus, $\pi_\theta(1\mid s_1)=\sigma(\theta)$ and $\pi_\theta(1\mid s_2)=\sigma(-\theta)$. The two instances differ only in their initial state distributions:
\[
\cM_+:d_0=\delta_{s_1},\quad \cM_-:d_0=\delta_{s_2}.
\]
We choose the respective comparator policies $\pi_{\cp,+}=\pi_{\log 3}$ and $\pi_{\cp,-}=\pi_{-\log 3}$, both achieving return $3/4$. Since $\gamma=0$, the critic is exact and independent of the policy, i.e., $f_k(s,a)=Q^{\pi_k}(s,a)=a$ for all $k$.

\begin{proof}[Proof of Proposition~\ref{thm:hardness}]

1) We first verify the regularity and the coverage condition for the above construction. The log-policy derivatives satisfy
\[
\nabla_\theta\log\pi_\theta(a\mid s)=x(s)(a-\pi_\theta(1\mid s)),\quad \nabla_\theta^2\log\pi_\theta(a\mid s)=-\pi_\theta(1\mid s)(1-\pi_\theta(1\mid s)).
\]
Therefore, $|\nabla_\theta\log\pi_\theta|\le 1$ and $|\nabla_\theta^2\log\pi_\theta|\le 1/4$, verifying Assumption~\ref{ass:policy} with $G=1$ and $\beta=1/4$. Moreover, the comparator occupancy is supported on $s_1$ in $\cM_+$ and $s_2$ in $\cM_-$, with action probabilities $3/4$ and $1/4$. Since $d^D(s,a)=1/4$, Assumption~\ref{ass:data} holds with $C=3$. In fact, all-policy density coverage holds with coefficient at most $4$.

2) We then show that any local update satisfying Definition~\ref{def:local-update} generates the same parameter-sequence distribution in both instances, regardless of initialization, which means \emph{any local update cannot distinguish the two instances}!

Fix any common parameter history $\theta_{1:k}$. By construction, $\cM_+$ and $\cM_-$ have identical offline distributions, policy parameterization, and exact critic functions. Therefore, their contextual information laws coincide, i.e., $\PP_k^{\cM_+}=\PP_k^{\cM_-}$. By Definition~\ref{def:local-update}, the conditional distributions of their update directions must also coincide, i.e., $\UU_k(\cdot\mid \theta_{1:k},\PP_k^{\cM_+})=\UU_k(\cdot\mid \theta_{1:k},\PP_k^{\cM_-})$.

This tells us that starting from any common initialization independent of the instance, induction shows that the entire parameter sequences $\theta_{1:K}$ have identical joint distributions in $\cM_+$ and $\cM_-$. Equivalently, we can couple their randomness so that the two sequences coincide almost surely.

3) Finally, this leads to regret lower bound. For any parameter $\theta$, the policy returns in the two instances are $J_{\cM_+}(\pi_\theta)=\sigma(\theta)$ and $J_{\cM_-}(\pi_\theta)=\sigma(-\theta)$. Since $\sigma(\theta)+\sigma(-\theta)=1$, and both comparators achieve return $3/4$, we have, for every $\theta$,
\[
\bigl[J_{\cM_+}(\pi_{\cp,+})-J_{\cM_+}(\pi_\theta)\bigr]+\bigl[J_{\cM_-}(\pi_{\cp,-})-J_{\cM_-}(\pi_\theta)\bigr]=\frac12
\]
Since the critics are exact and $\gamma=0$, by performance difference lemma, the comparator regret in Eq.~\eqref{eq:regret} can be represented by
\[
\EE_{\cM_+}\left[\frac{\Reg_K}{K}\right]+\EE_{\cM_-}\left[\frac{\Reg_K}{K}\right]=\frac{1}{K}\sum_{k=1}^K\frac12=\frac12.
\]
Consequently, $\max_{\cM\in\{\cM_+,\cM_-\}}\EE_{\cM}[\Reg_K/K]\ge 1/4=\Omega(1)$, $\forall K\ge 1$. Thus, at least one of the two indistinguishable instances incurs constant average comparator regret, regardless of the update rule or initialization.
\end{proof}

\paragraph{Relation to compatible function approximation.}

To align with the discussion of Section~\ref{sec:algorithm}, the constructed actor-critic pair above is \emph{incompatible}. Indeed, with
\[
A_\theta(s,a)=a-\pi_\theta(1\mid s),\quad\nabla_\theta\log\pi_\theta(a\mid s)=x(s)A_\theta(s,a),
\]
exact CFA would require the same update direction to satisfy $v=1$ at $s_1$ and $v=-1$ at $s_2$, which is impossible. The failure therefore persists despite exact critics and sufficient data coverage.

This also sheds light on the contextual coupling challenge. One source of the difficulty is that a shared policy parameter may fail to represent state-wise improvement directions consistently across states. Under compatible case, however, this obstruction can disappear. The following proposition revisits contextual mirror descent (Eq.~\eqref{eq:cmd}) in the \emph{compatible linear setting} and shows that, at the population level, it coincides with the LSPU update. 

\begin{proposition}[Contextual Mirror Descent under Exact Compatibility]\label{thm:cmd-compatibility}
    For invertible $\{\Sigma_D^\pi\}$, suppose $\pi_\theta(a\mid s)\propto\exp(\theta^\top\phi(s,a))$ and $f_k(s,a)=w_k^\top\phi(s,a)$ for some $\|w_k\|\le B_L$. Then the state-wise mirror descent update satisfies $\pi_{k+1}=\pi_{\theta_k+\eta w_k}$ and is therefore also a solution of Eq.~\eqref{eq:cmd}. Moreover, contextual mirror descent and population LSPU induce the same policy update.
\end{proposition}

\begin{proof}
Recall that for any fixed state $s$, the state-wise mirror descent update is $\pi_{k+1}(a\mid s)\propto \pi_k(a\mid s)\exp(\eta f_k(s,a)$. Since $\pi_k(a\mid s)\propto \exp(\theta_k^\top\phi(s,a))$ and $f_k(s,a)=w_k^\top\phi(s,a)+c_k(s)$, we have
\[
\pi_{k+1}(a\mid s)\propto \exp\!\big((\theta_k+\eta w_k)^\top\phi(s,a)\big),
\]
where the state-dependent term $\exp(\eta c_k(s))$ cancels under normalization. Hence $\pi_{k+1}=\pi_{\theta_k+\eta w_k}$. Since this policy maximizes the mirror-descent objective pointwise for every $s$ and belongs to $\Pi_\theta$, it also maximizes their $d^D$-weighted average in Eq.~\eqref{eq:cmd}.

It remains to connect this update to LSPU. For the log-linear policy,
\[
\nabla\log\pi_k(a\mid s)=\phi(s,a)-\EE_{a'\sim\pi_k(\cdot\mid s)}[\phi(s,a')],
\]
while
\[
A_k(s,a)=f_k(s,a)-f_k(s,\pi_k)=w_k^\top\nabla\log\pi_k(a\mid s).
\]
Therefore the population least-square loss satisfies $L_k(w_k)=0$. Since the feature covariance matrix $\Sigma_D^{\pi_k}$ is invertible, the population minimizer is unique, so LSPU chooses $v_k=w_k$. Its parameter update is therefore $\theta_{k+1}=\theta_k+\eta w_k$, which induces exactly the same policy update as contextual mirror descent.
\end{proof}

\subsection{More Discussion on the Hardness Construction}\label{app:challenge-more}

\paragraph{Data Admissibility.}

As shown in Appendix~\ref{app:challenge-proof}, our hardness construction considers two contextual bandit instances with identical offline contextual information but different initial distributions $d_0$, leading to conflicting policy improvements under a shared parameterization. This raises two concerns: first, an algorithm that explicitly leverages $d_0$ could distinguish the two instances and favor the appropriate improvement direction; second, the offline data distribution $d^D$ is \emph{inadmissible}, i.e., it cannot be induced by any behavior policy $\pi_b$ starting from $d_0$. Indeed, in contextual bandits, an admissible offline distribution must have its state marginal equal to $d_0$.

These observations raise the question of whether the hardness is merely due to the lack of access to $d_0$ or the inadmissibility of offline data. To address these concerns, we provide an MDP construction with known $d_0$ and admissible offline data $d^D=d^{\pi_b}$, under which deterministic local updates still suffer constant per-step regret. This result requires an additional natural property of the update rule, which we formalize below.

\begin{assumption}[Sign Symmetry]
\label{ass:sign}
For any parameter history $\theta_{1:k}$, contextual information law $\PP_k$, and measurable set $B\subseteq\mathbb \RR^\dd$, the update rule satisfies
\[
\UU_k(B\mid-\theta_{1:k},\PP_k^-)=\UU_k(-B\mid\theta_{1:k},\PP_k),
\]
where $\PP_k^-$ is obtained from $\PP_k$ by replacing the policy parameterization $\theta\mapsto\pi_\theta$ with $\theta\mapsto\pi_{-\theta}$, leaving all other contextual information unchanged.
\end{assumption}

Intuitively, sign symmetry requires that reversing the parameter history also reverses the conditional distribution of the update direction. This is a natural invariance property of gradient-based optimization methods, including contextual mirror descent, LSPU, and DRPU (with an equivariant optimizer-selection rule). Importantly, it is a property of the \emph{algorithm}, rather than the policy class.

\begin{proposition}[Failure of Deterministic Local Updates with Sign Symmetry]
    There exist an MDP $\cM$, a log-linear policy class $\Pi_\theta$ satisfying Assumption~\ref{ass:policy}, a comparator policy $\pi_\cp\in\Pi_\theta$, and an admissible offline distribution $d^D=d^{\pi_b}$ satisfying Assumption~\ref{ass:data}, such that any deterministic local update under Assumption~\ref{ass:sign}, initialized at $\theta_1=0$, incurs $\Reg_K/K=\Omega(1)$, $\forall K\ge 1$, even with infinite data and exact critics $f_k=Q^{\pi_k}$.
\end{proposition}

\begin{figure}
\vspace{-20pt}
\centering
\begin{tikzpicture}[
    >=Latex,
    state/.style={circle, draw, thick, minimum size=9mm, inner sep=1pt},
    edge/.style={->, thick},
    every node/.style={font=\small}
]
    \node[state] (s0) at (0,2.2) {$s_0$};
    \node[state] (s1) at (-1.4,0) {$s_1$};
    \node[state] (s2) at (1.4,0) {$s_2$};

    \node[above=2pt] at (s0.north) {$d_0=\delta_{s_0}$};

    \draw[edge] (s0) to[bend right=30]
        node[left, pos=0.55, align=center] {$a=0$\\$p=1$} (s1);

    \draw[edge] (s0) to[bend left=23]
        node[right, pos=0.65, align=center] {$a=1$\\$p=\tfrac14$} (s1);

    \draw[edge] (s0) --
        node[right, pos=0.58, align=center] {$a=1$\\$p=\tfrac34$} (s2);

    \draw[edge] (s1) edge[loop left, min distance=9mm] (s1);
    \draw[edge] (s2) edge[loop right, min distance=9mm] (s2);

    \node[font=\small] at (0,-0.85) {$R(s_1,a)=R(s_2,a)=a$};

\end{tikzpicture}
\vspace{-6pt}
\caption{MDP construction with known $d_0$ and admissible offline data $d^D$.}
\label{fig:mdp-hardness}
\end{figure}

\begin{proof}
    As shown in Figure~\ref{fig:mdp-hardness}, consider an MDP with $\cS=\{s_0,s_1,s_2\}$, $\cA=\{0,1\}$, $\gamma=3/4$ and $d_0=\delta_{s_0}$. At $s_0$, action $0$ transits to $s_1$, while action $1$ transits to $s_1$ and $s_2$ with probability $1/4$ and $3/4$, respectively, Both $s_1$ and $s_2$ are absorbing, with rewards $R(s_i,a)=a$, while $R(s_0,a)=0$.

    Consider the one-dimensional log-linear policy class
    \[
    \pi_\theta(1\mid s_0)=\frac12,\quad \pi_\theta(1\mid s_1)=\sigma(\theta),\quad \pi_\theta(1\mid s_2)=\sigma(-\theta).
    \]
    Let $d^D=d^{\pi_b}$ be induced by a behavior policy satisfying $\pi_b(1\mid s_0)=2/3$ and $\pi_b(1\mid s_i)=3/4$ for $i=1,2$. Thus, $d^D$ is admissible with state marginal $d^D(s_0)=1/4$ and $d^D(s_1)=d^D(s_2)=3/8$.

    Initialize $\theta_1=0$. Since the critic is exact, we have $f_1(s_0,a)=3/2$ and $f_1(s_i,a)=a+3/2$, $i=1,2$. Under the sign reparameterization $\theta\mapsto-\theta$, the policy parameterization at $s_0$ remains unchanged, while those at $s_1$ and $s_2$ are changed, i.e.,
    \[
    \pi_{-\theta}(\cdot\mid s_1)=\pi_\theta(\cdot\mid s_2),\quad \pi_{-\theta}(\cdot\mid s_2)=\pi_\theta(\cdot\mid s_1).
    \]
    Since $s_1$ and $s_2$ have the same critic values, conditional offline action distributions, and offline state masses, therefore exchanging their contextual information leaves the joint distribution $\PP_1$ unchanged, i.e., $\PP_1^-=\PP_1$, satisfying the sign symmetry condition in Assumption~\ref{ass:sign}.

    Now we consider any deterministic update $v_1=\UU_1(0,\PP_1)$ as $\theta_1=0$. By sign symmetry and the equality $\PP_1^-=\PP_1$,
    \[
    v_1=\UU_1(0,\PP_1)=\UU_1(0,\PP_1^-)=-\UU_1(0,\PP_1)=-v_1.
    \]
    Hence $v_1=0$ and $\theta_2=0$. Inductively, if $\theta_1=\cdots=\theta_k=0$, the exact critics remain unchanged, so $\PP_k^-=\PP_k$ and the same argument implies $v_k=0$. Consequently, $\theta_k=0$ for all $k\ge 1$.

    For any policy $\pi_\theta$, its performance can be calculated as
    \[
    J(\pi_\theta)=\frac{\gamma}{1-\gamma}\left(\frac58\sigma(\theta)+\frac38\sigma(-\theta)\right)=\frac98+\frac34\sigma(\theta).
    \]
    Therefore, $J(\pi_0)=3/2$, and by choosing $\theta_\cp=\log 3$, we have $J(\pi_\cp)=27/16$. Since the critics are exact, by performance difference lemma, we have
    \[
    \frac{\Reg_K}{K}=\frac{1-\gamma}{K}\sum_{k=1}^K\bigl(J(\pi_\cp)-J(\pi_k)\bigr)=\frac{3}{64}=\Omega(1).
    \]

\end{proof}

\section{Omitted Details for Section~\ref{sec:npg}}\label{app:npg}

In this section, we establish the performance guarantee of the LSPU update. We will first prove Lemma~\ref{lemma:NPG-regret}, and then decompose the error of compatible function approximation (CFA) into the bias term $\epsilon_\bias$ and the estimation error term $\epsilon_\stat$. And then we utilize results from \citet{hsu2012random} to give a sharp analysis on linear regression, which controls the estimation error $\epsilon_\stat$. Finally, we will give a more rigorous analysis that is based on log-covering number (i.e., metric entropy) rather than log-cardinality of the function or policy class.

\subsection{Proof of Lemma~\ref{lemma:NPG-regret}}\label{app:challenge-decomposition}
\begin{proof}[Proof of Lemma~\ref{lemma:NPG-regret}]
With Assumption~\ref{ass:policy}, $\log\pi_\theta(a\mid s)$ is $\beta$-smooth in $\theta$ with respect to some norm $\|\cdot\|$. 
That is, for all $(s,a)\in\cS\times\cA$, we can bound the first-order approximation error between arbitrary $\theta,\theta'\in\RR^\dd$ as
\[
-\frac{\beta}{2}\|\theta'-\theta\|^2
\le
\log\pi_{\theta'}(a\mid s)-\log\pi_\theta(a\mid s)
-\langle\nabla\log\pi_\theta(a\mid s),\theta'-\theta\rangle
\le
\frac{\beta}{2}\|\theta'-\theta\|^2.
\]
Applying the update rule $\theta_{k+1}=\theta_k+\eta v_k$ and substituting $\theta'=\theta_{k+1}$, $\theta=\theta_k$, we obtain (by taking the negative side):
\begin{align*}
\log\frac{\pi_{k+1}(a\mid s)}{\pi_k(a\mid s)}&\geq\left(\nabla_\theta\log\pi_k(a\mid s)\right)^\top (\theta_{k+1}-\theta_k)-\frac{\beta}{2}\|\theta_{k+1}-\theta_k\|^2\\
&\geq\eta\left(\nabla_\theta\log\pi_k(a\mid s)\right)^\top v_k-\frac{\beta}{2}\eta^2\|v_k\|^2.
\end{align*}
Recall the definition of $\err_k=\EE_{(s,a)\sim d^{\pi_\cp}}\left[A_k(s,a)-\left(\nabla_\theta\log\pi_k(a\mid s)\right)^\top v_k\right]$, we have
\[
\EE_{(s,a)\sim d^{\pi_\cp}}\left[\left(\nabla_\theta\log\pi_k(a\mid s)\right)^\top v_k\right]=\EE_{(s,a)\sim d^{\pi_\cp}}\left[A_k(s,a)\right]-\err_k.
\]
Hence, by the definition of KL divergence $\KL(p\|q)=\EE_{p}[\log(p/q)]$, we have
\begin{align*}
&\EE_{s\sim d^{\pi_\cp}}\left[\KL(\pi_\cp(\cdot\mid s)\|\pi_k(\cdot\mid s))-\KL(\pi_\cp(\cdot\mid s)\|\pi_{k+1}(\cdot\mid s))\right]\\
&=\EE_{s\sim d^{\pi_\cp}}\left[\EE_{a\sim\pi_\cp(\cdot\mid s)}\left[\log\frac{\pi_{k+1}(a\mid s)}{\pi_k(a\mid s)}\right]\right]\\
&\geq\EE_{(s,a)\sim d^{\pi_\cp}}\left[\eta\left(\nabla_\theta\log\pi_k(a\mid s)\right)^\top v_k-\frac{\beta}{2}\eta^2\|v_k\|^2\right]\\
&\geq\eta\EE_{(s,a)\sim d^{\pi_\cp}}\left[A_k(s,a)\right]-\eta\cdot \err_k-\frac{\beta}{2}\eta^2\|v_k\|^2,
\end{align*}
where the first inequality is due to the previous display of the smoothness, and the second inequality is due to the rewrite of $\err_k$. Since $A_k(s,a)=f_k(s,a)-f_k(s,\pi_k)$, rearranging the terms, we get
\begin{align*}
\EE_{s\sim d^{\pi_\cp}}\bigl[f_k(s,\pi_\cp)-&f_k(s,\pi_k)\bigr]=\EE_{(s,a)\sim d^{\pi_\cp}}[A_k(s,a)]\leq\err_k+\frac{\beta}{2}\eta \|v_k\|^2\\
&+\frac{1}{\eta}\EE_{s\sim d^{\pi_\cp}}\left[\KL(\pi_\cp(\cdot\mid s)\|\pi_k(\cdot\mid s))-\KL(\pi_\cp(\cdot\mid s)\|\pi_{k+1}(\cdot\mid s))\right].
\end{align*}
Telescoping the last term of KL divergence, we get
\[
\Reg_K\leq\sum_{k=1}^K\err_k+\eta\cdot \frac{\beta}{2}\sum_{k=1}^K\|v_k\|^2+\frac{1}{\eta}\cdot \EE_{s\sim d^{\pi_\cp}}\left[\KL(\pi_\cp(\cdot\mid s)\|\pi_1(\cdot\mid s))\right].
\]
Since the update direction $v_k$ satisfies $\|v_k\|\le B_L$ for all $k\in[K]$, choosing the step size optimally as $\eta=\sqrt{\frac{2\KL(\pi_\cp\|\pi_1)}{\beta K B_L^2}}$,
yields the regret bound
\[
\frac{\Reg_K}{K}\le \frac{1}{K}\sum_{k=1}^K\err_k
+B_L\sqrt{\frac{2\beta\cdot\KL(\pi_\cp\|\pi_1)}{K}},
\]
where $\KL(\pi_\cp\|\pi_1)
=\EE_{s\sim d^{\pi_\cp}}[\KL(\pi_\cp(\cdot\mid s)\|\pi_1(\cdot\mid s))]$ is the expected KL divergence under the comparator occupancy $d^{\pi_\cp}$.
\end{proof}

\paragraph{Step-size selection.} The step size in Lemma~\ref{lemma:NPG-regret} depends on the unknown comparator policy $\pi_\cp$ through $\KL(\pi_\cp\|\pi_1)$. However, this dependence is not essential. Since by Lemma~\ref{lemma:NPG-regret}, we have
\[
\frac{\Reg_K}{K}\le\frac{\KL(\pi_\cp\|\pi_1)}{\eta K}+\frac{\beta \eta B_L^2}{2}+\frac{1}{K}\sum_{k=1}^K\err_k.
\]
Consequently, given any known upper bound $\bar D\geq \KL(\pi_\cp\|\pi_1)$, choosing $\eta=\sqrt{2\bar D/(\beta K B_L^2)}$ yields the same regret bound with $\KL(\pi_\cp\|\pi_1)$ replaced by $\bar D$. For example, when the action space is finite and $\pi_1$ is uniform, we may take $\bar D=\log|\cA|$. Other structural bounds on the initial KL divergence can also be used (see Appendix~\ref{app:continuous-ex}). Furthermore, choosing a step size $\eta=c\eta^\star$ for any constant $c>0$, where $\eta^\star$ denotes the theoretically optimal choice, increases the optimization term by only a factor of $(c+c^{-1})/2$, preserving the $\cO(K^{-1/2})$ rate.

\subsection{Two Cases with Zero Actor-Critic Incompatibility}\label{app:nobias}
Taking a detour, we first show that the intrinsic bias term $\epsilon_\bias$ vanishes in two special cases: 
(i) under the tabular softmax policy class with any function class $\cF$, and 
(ii) in the linear function approximation setting \citep{jin2020provably,yang2019sample,jiang2017contextual} under the log-linear policy class. In general, if we have zero actor-critic incompatibility, then $\epsilon_\bias=0$.
\begin{proposition}[No Bias under Tabular Softmax Policy]\label{lemma:nobias-1}
Consider the softmax policy class
\[
\pi_\theta(a\mid s)
=\frac{\exp(\theta(s,a))}{\sum_{a'\in\cA}\exp(\theta(s,a'))},
\]
where $\theta(s,a)$ denotes the $(s,a)$-th component of the parameter $\theta\in\RR^{\cS\times\cA}$. 
Then, for any advantage function $A_k(s,a)$, there exists a parameter update direction $v_k^{\textup{n}}$ such that
\[
A_k(s,a) = (v_k^{\textup{n}})^\top\nabla_\theta \log \pi_k(a\mid s)
\quad\text{for all }(s,a)\in\cS\times\cA.
\]
Consequently, the function approximation is well-specified, 
and the intrinsic bias term in Assumption~\ref{ass:npg} satisfies $\epsilon_\bias=0$.
\end{proposition}

\begin{proof}
Under the softmax parameterization, the log-policy gradient takes the form
\[
\nabla_\theta \log \pi_k(a\mid s)
= \nabla_\theta \theta(s,a)
- \EE_{a'\sim\pi_k(\cdot\mid s)}[\nabla_\theta \theta(s,a')].
\]
Since $\theta(s,a)$ can be any real-valued function over $(s,a)$,
the gradient basis $\nabla_\theta \log \pi_k(a\mid s)$ spans all zero-mean functions 
with respect to $\pi_k(\cdot\mid s)$.
The advantage function $A_k(s,a)$ also satisfies 
$\EE_{a\sim\pi_k(\cdot\mid s)}[A_k(s,a)] = 0$ by definition.
Therefore, $A_k$ lies exactly in the span of 
$\nabla_\theta \log \pi_k(a\mid s)$,
implying that there exists a vector $v_k^{\text{n}}$ achieving zero regression error.
Moreover, any function of the form $A_k(s,a)+c_k$, where $c_k$ is a state-dependent constant, 
is also a minimizer, since
\[
\EE_{a\sim \pi_k(\cdot\mid s)}[\nabla_\theta\log\pi_k(a\mid s)]
= \sum_{a\in\cA}\pi_k(a\mid s)\nabla_\theta\log\pi_k(a\mid s)
= \nabla_\theta \sum_{a\in\cA} \pi_k(a\mid s)
= \nabla_\theta 1 = 0.
\]
Hence, under the canonical softmax policy class, we have $\epsilon_\bias = 0$.
\end{proof}

\begin{proposition}[No Bias under Compatible Log-Linear Policy]\label{lemma:nobias-loglinear}
Assume a log-linear policy
\[
\pi_\theta(a\mid s)=\frac{\exp(\theta^\top\phi(s,a))}{\sum_{a'\in\cA}\exp(\theta^\top\phi(s,a'))},
\]
and a linear function class $\cF=\{f_w(s,a)=\phi(s,a)^\top w:w\in\RR^\dd\}$ with the same feature map $\phi:\cS\times\cA\to\RR^\dd$. At iteration $k$, let $f_k\in\cF$ and define $A_k(s,a)$. Then there exists a parameter update direction $v_k^{\textup{n}}\in\RR^\dd$ such that
\[
A_k(s,a)=v_k^{\textup{n}\top}\nabla_\theta \log \pi_k(a\mid s)\quad\text{for all }(s,a)\in\cS\times\cA,
\]
and hence the function approximation is realizable with $\epsilon_\bias=0$ in Assumption~\ref{ass:npg}.
\end{proposition}

\begin{proof}
Fix an iteration $k$ and a state $s$. Define $Z(s,\theta)=\sum_{a'\in\cA}\exp(\theta^\top\phi(s,a'))$. 
Then $\log\pi_\theta(a\mid s)=\theta^\top\phi(s,a)-\log Z(s,\theta)$, whose gradient is
\[
\nabla_\theta\log\pi_\theta(a\mid s)=\phi(s,a)-\EE_{a'\sim\pi_\theta(\cdot\mid s)}[\phi(s,a')].
\]
Substituting $\theta=\theta_k$ gives $\nabla_\theta\log\pi_k(a\mid s)$. On the critic side, $f_k(s,a)=\phi(s,a)^\top w_k$ for some $w_k\in\RR^\dd$, so the advantage
\[
A_k(s,a)=f_k(s,a)-f_k(s,\pi_k)
=w_k^\top\left(\phi(s,a)-\EE_{a'\sim\pi_k(\cdot\mid s)}[\phi(s,a')]\right).
\]
Comparing both expressions shows $A_k(s,a)=w_k^\top\nabla_\theta\log\pi_k(a\mid s)$ for all $(s,a)$. 
Hence the regression model is exactly realizable with $v_k^{\textup{n}}=w_k$, producing zero residual and $\epsilon_\bias=0$.
\end{proof}

\subsection{Decomposed Regret Bound and Bounding the Estimation Error}\label{app:lspu-proof-prep}
In this section we prepare for proving Theorem~\ref{thm:npg-guarantee}. First, we can decompose the CFA error into a bias term $\epsilon_\bias$ (in Assumption~\ref{ass:npg}) and an estimation error term $\epsilon_\stat$, multiplied by some coverage constant. Recall that the least-square loss $L_k$ at round $k$ is given by
\[
L_k(v)=\mathbb{E}_{(s,a)\sim d^D}\bigl[(A_k(s,a)-v^\top\nabla\log\pi_k(a\mid s))^2\bigr].
\]
And the bias term and the estimation error term in Section~\ref{sec:npg} are respectively defined as
\[
L_k(v_k^*)=\min_{v:\|v\|\leq B_L}L_k(v)\leq\epsilon_\bias,\qquad L_k(v_k)-L_k(v_k^*)\leq\epsilon_\stat,\qquad \forall k\in[K].
\]
The following lemma formalizes this fact of decomposition.

\begin{lemma}[Decomposed Regret Bound for LSPU]
\label{lemma:npg-decomposed-regret}
Under Assumptions~\ref{ass:data}, \ref{ass:policy}, and~\ref{ass:npg}, the update in Eq.~\eqref{eq:update} with step size $\eta=\sqrt{2\KL(\pi_\cp\|\pi_1)/(\beta K B_L^2)}$ achieves the following regret bound:
\[
\frac{\Reg_K}{K} \leq B_L \sqrt{\frac{2\beta \KL(\pi_\cp\|\pi_1)}{K}} + \sqrt{\CC}\Bigl(\sqrt{\epsilon_\bias} + \sqrt{\epsilon_\stat}\Bigr).
\]
\end{lemma}

\begin{proof}
Using the regret decomposition lemma (Lemma~\ref{lemma:NPG-regret}), we have
\[
\frac{\Reg_K}{K} \leq B_L \sqrt{\frac{2\beta \KL(\pi_\cp\|\pi_1)}{K}} +\frac{1}{K}\sum_{k=1}^K\err_k.
\]
We make the following decomposition of $\err_k$:
\[
\err_k=\EE_{(s,a)\sim d^{\pi_\cp}}\left[A_k(s,a)-(v_k^*)^\top\nabla_\theta\log\pi_k(a\mid s)\right]+\EE_{(s,a)\sim d^{\pi_\cp}}\left[(v_k^*-v_k)^\top\nabla_\theta\log\pi_k(a\mid s)\right].
\]
By the coverage assumption that $\|d^{\pi_\cp}/d^D\|_\infty\leq \CC$, we can translate the error to the offline data distribution $d^D$ and bound the first term with
\begin{align*}
\EE_{(s,a)\sim d^{\pi_\cp}}\left[A_k(s,a)-(v_k^*)^\top\nabla_\theta\log\pi_k(a\mid s)\right]
&\leq\sqrt{\EE_{(s,a)\sim d^{\pi_\cp}}\left[\left(A_k(s,a)-(v_k^*)^\top\nabla_\theta\log\pi_k(a\mid s)\right)^2\right]}\\
&\leq\sqrt{\CC\cdot\EE_{(s,a)\sim d^{D}}\left[\left(A_k(s,a)-(v_k^*)^\top\nabla_\theta\log\pi_k(a\mid s)\right)^2\right]}\\
&=\sqrt{\CC\cdot L_k(v_k^*)}\leq\sqrt{\CC\cdot\epsilon_\bias}.
\end{align*}
For the second term, we note that $v_k^*$ is the minimizer of $L_k(v)$ over the set $\mathcal{V}=\{v:\|v\|\leq B_L\}$ (since we assume $\|v_k^*\|\le B_L$ in Assumption~\ref{ass:npg}). Hence for any $v$ such that $\|v\|\leq B_L$, the first-order optimality condition for $v_k^*$ imply that
\[
\left(v-v_k^*\right)^\top\nabla_v L_k(v_k^*)\geq 0.
\]
Therefore, for any $v$ such that $\|v\|\leq B_L$, we have
\begin{align*}
L_k(v)-L_k(v_k^*)&=\EE_{d^D}\left[(A_k-\phi_k^\top v)^2\right]-\EE_{d^D}\left[(A_k-\phi_k^\top v_k^*)^2\right]\\
&=\EE_{d^D}\left[(A_k-\phi_k^\top v_k^*+\phi_k^\top v_k^*-\phi_k^\top v)^2\right]-\EE_{d^D}\left[(A_k-\phi_k^\top v_k^*)^2\right]\\
&=\EE_{d^D}\left[\left(\phi_k^\top(v_k^*-v)\right)^2\right]-2\EE_{d^D}\left[\phi_k^\top(v-v_k^*)(A_k-\phi_k^\top v_k^*)\right]\\
&=\|v_k^*-v\|_{\Sigma_{D}}^2+(v-v_k^*)^\top\nabla_v L_k(v_k^*)\\
&\geq \|v_k^*-v\|_{\Sigma_D}^2,
\end{align*}
where we use $\phi_k$ to denote the feature $\nabla_\theta\log\pi_k$, and $\Sigma_D=\EE_{d^D}[\phi_k\phi_k^\top]$. The last inequality is due to the first-order optimality condition stated before. By taking $v=v_k$ and the coverage condition, we have
\begin{align*}
\EE_{(s,a)\sim d^{\pi_\cp}}\left[(v_k^*-v_k)^\top\nabla_\theta\log\pi_k(a\mid s)\right]&\leq \sqrt{\EE_{(s,a)\sim d^{\pi_\cp}}\left[\left((v_k^*-v_k)^\top\nabla_\theta\log\pi_k(a\mid s)\right)^2\right]}\\
&\leq \sqrt{\CC\cdot \EE_{(s,a)\sim d^D}\left[\left((v_k^*-v_k)^\top\nabla_\theta\log\pi_k(a\mid s)\right)^2\right]}\\
&=\sqrt{\CC\cdot\|v_k^*-v_k\|_{\Sigma_D}^2}\\
&\leq\sqrt{\CC\cdot\left(L_k(v_k)-L_k(v_k^*)\right)}\leq\sqrt{\CC\cdot \epsilon_\stat}.
\end{align*}
Combining those two terms, we get that $\err_k\leq\sqrt{\CC}(\sqrt{\epsilon_\bias}+\sqrt{\epsilon_\stat})$. Substituting into the regret bound, we get Lemma~\ref{lemma:npg-decomposed-regret}:
\[
\frac{\Reg_K}{K} \leq B_L \sqrt{\frac{2\beta \KL(\pi_\cp\|\pi_1)}{K}} + \sqrt{\CC}\Bigl(\sqrt{\epsilon_\bias} + \sqrt{\epsilon_\stat}\Bigr). \tag*{\qedhere}
\]
\end{proof}

We establish a single uniform event for the actor dataset, allowing the same $N$ i.i.d. pairs from $d^D$ to be reused at every round. Throughout this analysis, norms are Euclidean (and matrix norms are the induced operator norms). Write $\EE=\EE_{d^D}$ and $\widehat{\EE}$ for its empirical average, and, for a critic--parameter pair, define
\[
A_{f,\theta}=f-f(\cdot,\pi_\theta),\qquad \phi_\theta=\nabla\log\pi_\theta,\qquad \Sigma_\theta=\EE[\phi_\theta\phi_\theta^\top].
\]
Fix a deterministic initialization and step size, and let $\Theta_0=\{\theta:\|\theta-\theta_1\|_2\leq\eta B_LK\}$. Bounded updates imply that every iterate lies in this region; no projection of the policy parameter is required. Assumption~\ref{ass:policy} applies on this region and its line segments. We assume $\KL(\pi_\cp\|\pi_1)<\infty$.

\begin{assumption}[Data Model for Uniform Linear Regression]\label{ass:linear-regression}
Fix a deterministic $\lambda_\tmin>0$ and let $\Theta_\lambda=\{\theta\in\Theta_0:\Sigma_\theta\succeq\lambda_\tmin I\}$. There is a deterministic class $\mathcal{U}\subseteq\cF\times\Theta_\lambda$ containing every generated pair $(f_k,\theta_k)$. For every $(f,\theta)\in\mathcal{U}$, the unconstrained population solution and its residual satisfy
\begin{gather*}
v^*_{f,\theta}=\Sigma_\theta^{-1}\EE[\phi_\theta A_{f,\theta}],\qquad
\epsilon_{f,\theta}=A_{f,\theta}-(v^*_{f,\theta})^\top\phi_\theta,\\
\|v^*_{f,\theta}\|_2\leq B_L,\qquad \EE[\epsilon_{f,\theta}^2]\leq\epsilon_\bias,\qquad \|\epsilon_{f,\theta}\|_\infty\leq b.
\end{gather*}
Here $b\leq B:=V_\tmax+GB_L$ is deterministic; taking $b=B$ always suffices. All classes are measurable and separable in the displayed uniform metrics, and the covering numbers used below are finite.
\end{assumption}
This is the uniform form of the feasible unconstrained optimum in Assumption~\ref{ass:npg}. It imposes nondegeneracy only on the admissible parameter set, not on every policy. Normal equations give $\EE[\phi_\theta\epsilon_{f,\theta}]=0$. The advantage has no additional Monte Carlo label noise; the argument does not assume it is determined by the score feature alone.

\begin{lemma}[Estimation Error Decomposition]\label{lemma:ols-decomposition}
Fix an admissible pair, suppress its subscripts, and suppose $\widehat\Sigma\succeq\Sigma/2$. If $\widehat v$ is feasible and its empirical squared loss is within $\zeta\geq0$ of the constrained minimum, then
\[
\|\widehat v-v^*\|_\Sigma\leq\frac{4}{\sqrt{\lambda_\tmin}}\|\widehat{\EE}[\phi\epsilon]\|_2+\sqrt{2\zeta},\qquad
L(\widehat v)-L(v^*)=\|\widehat v-v^*\|_\Sigma^2.
\]
\end{lemma}
\begin{proof}
Let $\Delta=\widehat v-v^*$ and $q=\widehat{\EE}[\phi\epsilon]$. Feasibility of $v^*$ and empirical near-optimality imply $\|\Delta\|_{\widehat\Sigma}^2\leq2\Delta^\top q+\zeta$. Thus $\|\Delta\|_{\widehat\Sigma}\leq2\|q\|_{\widehat\Sigma^{-1}}+\sqrt\zeta$. Since $\|\Delta\|_\Sigma\leq\sqrt2\|\Delta\|_{\widehat\Sigma}$ and $\|q\|_{\widehat\Sigma^{-1}}\leq\sqrt{2/\lambda_\tmin}\|q\|_2$, the first bound follows. Expanding the population loss and using $\EE[\phi\epsilon]=0$ proves the equality.
\end{proof}

\begin{lemma}[Uniform Covariance Bound]\label{lemma:linear-regression}
Let $M_\Sigma$ be the size of a Euclidean $\lambda_\tmin/(16G\beta)$-cover of $\Theta_\lambda$ with centers in that set; when $\beta=0$, take $M_\Sigma=1$. If
\begin{equation}\label{eq:lspu-burnin}
N\geq\frac{32G^2}{\lambda_\tmin}\log\frac{2\dd M_\Sigma}{\delta},
\end{equation}
then, with probability at least $1-\delta/2$, $\widehat\Sigma_\theta\succeq\Sigma_\theta/2$ simultaneously for every $\theta\in\Theta_\lambda$.
\end{lemma}
\begin{proof}
At a fixed center $\widetilde\theta$, the whitened feature $Y=\Sigma_{\widetilde\theta}^{-1/2}\phi_{\widetilde\theta}$ satisfies $\EE[YY^\top]=I$ and $\|Y\|_2^2\leq G^2/\lambda_\tmin$. Lemma~\ref{lemma:matrix-chernoff} gives failure probability at most $\dd\exp\{-N\lambda_\tmin/(32G^2)\}$ for $\widehat\Sigma_{\widetilde\theta}\succeq3\Sigma_{\widetilde\theta}/4$. Union-bound over the centers. For $\|\theta-\widetilde\theta\|_2\leq r$, smoothness gives
\[
\|\phi_\theta\phi_\theta^\top-\phi_{\widetilde\theta}\phi_{\widetilde\theta}^\top\|_{\rm op}\leq2G\beta r.
\]
This bounds both covariance perturbations, so $\widehat\Sigma_\theta\succeq3\Sigma_\theta/4-(7/2)G\beta r I\succeq\Sigma_\theta/2$ at the stated radius. When $\beta=0$, the score is constant on the connected region $\Theta_0$, so one center suffices.
\end{proof}

\subsection{Proof of Theorem~\ref{thm:npg-guarantee}}\label{app:lspu-proof}
\begin{definition}[$\varepsilon$-covering number]\label{def:covering-number}
An $\varepsilon$-cover of a metric space $(\mathcal{G},d)$ is a finite subset of $\mathcal{G}$ within distance $\varepsilon$ of every member. Its smallest cardinality is $\cN_d(\mathcal{G},\varepsilon)$. For critics we use the uniform norm; for parameters we use the parameter norm; for vector-valued functions we use $\|h\|_{\infty,2}=\sup_{s,a}\|h(s,a)\|_2$.
\end{definition}
Define the deterministic centered moment class and its entropy by
\begin{equation}\label{eq:complexity}
\mathcal{H}_L=\{\phi_\theta\epsilon_{f,\theta}:(f,\theta)\in\mathcal{U}\},\qquad
\xi=GB/N^2,\qquad
\Comp_L=\log\frac{2\cN_{\infty,2}(\mathcal{H}_L,\xi)}{\delta}.
\end{equation}
Every member $h$ satisfies $\EE[h]=0$, $\|h\|_{\infty,2}\leq Gb$, and $\EE\|h\|_2^2\leq G^2\epsilon_\bias$. Set
\begin{equation}\label{eq:lspu-T}
T_L=G\sqrt{\frac{\epsilon_\bias}{N}}(1+\sqrt{8\Comp_L})+\frac{4Gb\Comp_L}{3N}+\xi.
\end{equation}

\begin{lemma}[Uniform LSPU Estimation Error]\label{lemma:linear-regression-cov}
Under Assumptions~\ref{ass:policy}, \ref{ass:linear-regression} and Eq.~\eqref{eq:lspu-burnin}, with probability at least $1-\delta$, simultaneously for all admissible pairs and all empirical $\zeta$-minimizers,
\[
\sqrt{L(\widehat v)-L(v^*)}\leq\frac{4T_L}{\sqrt{\lambda_\tmin}}+\sqrt{2\zeta}.
\]
\end{lemma}
\begin{proof}
Apply Lemma~\ref{lemma:matrix-bernstein} at each center of the deterministic $\xi$-cover of $\mathcal H_L$, with vector bound $Gb$ and total second moment $NG^2\epsilon_\bias$. A union bound gives, with probability at least $1-\delta/2$,
\[
\sup_{h\in\mathcal H_L}\|\widehat{\EE}[h]\|_2\leq G\sqrt{\frac{\epsilon_\bias}{N}}(1+\sqrt{8\Comp_L})+\frac{4Gb\Comp_L}{3N}+\xi=T_L.
\]
The last term extends the bound from the cover to the whole class. Intersect with Lemma~\ref{lemma:linear-regression}, then apply Lemma~\ref{lemma:ols-decomposition}. This event is uniform before the data-dependent critics and policies are selected; empirical optimizers are never transferred to a different regression problem.
\end{proof}

\paragraph{Explicit critic and policy entropy.}
The complexity in Eq.~\eqref{eq:complexity} can be bounded using separate critic and parameter covers. For admissible pairs within distances $\varepsilon_f,\varepsilon_\theta$, differentiating the policy along the parameter segment gives $\sup_s\|\pi_\theta(\cdot|s)-\pi_{\widetilde\theta}(\cdot|s)\|_1\leq G\varepsilon_\theta$. Consequently,
\[
\|A-\widetilde A\|_\infty\leq a_0:=2\varepsilon_f+V_\tmax G\varepsilon_\theta,\qquad
\|\phi-\widetilde\phi\|_{\infty,2}\leq\beta\varepsilon_\theta,\qquad
\|\Sigma-\widetilde\Sigma\|_{\rm op}\leq2G\beta\varepsilon_\theta.
\]
The normal equations yield $\Sigma(v^*-\widetilde v^*)=\EE[\phi A-\widetilde\phi\widetilde A]+(\widetilde\Sigma-\Sigma)\widetilde v^*$, hence
\[
\|v^*-\widetilde v^*\|_2\leq\frac{Ga_0+\beta(V_\tmax+2GB_L)\varepsilon_\theta}{\lambda_\tmin}=:d_v,\qquad
\|\epsilon-\widetilde\epsilon\|_\infty\leq a_0+B_L\beta\varepsilon_\theta+Gd_v.
\]
Thus $\|h-\widetilde h\|_{\infty,2}\leq L_f\varepsilon_f+L_\theta\varepsilon_\theta$, where
\[
L_f=2G(1+G^2/\lambda_\tmin),\qquad
L_\theta=b\beta+G\left[V_\tmax G+B_L\beta+\frac{G}{\lambda_\tmin}(V_\tmax G^2+\beta V_\tmax+2GB_L\beta)\right].
\]
Cover $\cF$ and $\Theta_\lambda$ at radii $\varepsilon_f/2$ and $\varepsilon_\theta/2$, and select one admissible representative from each nonempty product cell. Representatives remain in $\mathcal U$, so their moments are centered and obey the same bounds. Taking $\varepsilon_f=\xi/(2L_f)$ and $\varepsilon_\theta=\xi/(2L_\theta)$ gives
\[
\cN_{\infty,2}(\mathcal H_L,\xi)\leq\cN_\infty(\cF,\varepsilon_f/2)\cN_2(\Theta_\lambda,\varepsilon_\theta/2).
\]
A zero coefficient requires no corresponding cover. These formulas specify the metrics and scales; boundedness of $\cF$ alone does not imply finite uniform entropy.

\begin{proof}[Proof of Theorem~\ref{thm:npg-guarantee}]
On the event of Lemma~\ref{lemma:linear-regression-cov}, density coverage and Cauchy--Schwarz give, simultaneously over rounds,
\[
\err_k\leq\sqrt{\CC\epsilon_\bias}+\sqrt\CC\|v_k-v_k^*\|_{\Sigma_{\theta_k}}
\leq\sqrt{\CC\epsilon_\bias}+4\sqrt{\frac{\CC}{\lambda_\tmin}}T_L+\sqrt{2\CC\zeta_k}.
\]
For any fixed $\eta>0$, the KL-potential calculation in Appendix~\ref{app:challenge-decomposition} therefore gives the explicit bound
\begin{equation}\label{eq:lspu-full}
\frac{\Reg_K}{K}\leq\frac{\KL(\pi_\cp\|\pi_1)}{\eta K}+\frac{\beta\eta B_L^2}{2}+\sqrt{\CC\epsilon_\bias}+4\sqrt{\frac{\CC}{\lambda_\tmin}}T_L+\frac1K\sum_{k=1}^K\sqrt{2\CC\zeta_k}.
\end{equation}
With exact empirical minimization ($\zeta_k=0$), the step size of Lemma~\ref{lemma:NPG-regret}, $b\leq B$, and Eq.~\eqref{eq:lspu-T}, this proves the main-text bound. Both the $\Comp_L/N$ and the covering remainder are retained. If $\epsilon_\bias=0$, separability makes all admissible residuals zero on a common set of full $d^D$ measure. Then $\widehat{\EE}[\phi\epsilon]=0$ simultaneously, and exact empirical minimization on the uniform covariance event gives $v_k=v_k^*$ and zero statistical error.
\end{proof}

\section{Omitted Details for Section~\ref{sec:dro}}\label{app:dro}

Using the generic regret decomposition lemma (Lemma~\ref{lemma:NPG-regret}), we can bound the regret term as (by taking $B_L=V_\tmax$):
\begin{equation}\label{eq:dro-regret}
\frac{\Reg_K}{K}\le V_\tmax\sqrt{\frac{2\beta \KL(\pi_\cp\|\pi_1)}{K}}+\widetilde{\epsilon}_\bias+\widetilde{\epsilon}_\stat,
\end{equation}
where $\widetilde{\epsilon}_\bias$ and $\widetilde{\epsilon}_\stat$ are similarly the approximation error and the estimation error. Recall:
\[
\ell_k(\tilde{v}_k^*)=\min_{v:\|v\|\leq V_\tmax}\ell_k(v)\le \widetilde{\epsilon}_\bias,\qquad \ell_k(v_k)-\ell_k(\tilde{v}_k^*)\le\widetilde{\epsilon}_\stat,\qquad \forall k\in[K].
\]
Since $\widetilde{\epsilon}_\bias$ is given in Assumption~\ref{ass:dro}, our goal is to give a non-asymptotic control of $\widetilde{\epsilon}_\stat$.

\subsection{Proof of Theorem~\ref{thm:dro-guarantee}}\label{app:dro-proof}
We first show that under $\cW_\infty$ class, the robust optimization problem can essentially be transformed as a CVaR problem, which would validate that DRPU under $\cW_\infty$ in Eq.~\eqref{eq:weight} by setting $\alpha=1/\CC$.
\begin{proposition}\label{prop:cvar}
Let $Z$ be an integrable random variable and $\alpha\in(0,1]$ be a tail probability level.
Then the following equivalence holds:
\[
\sup_{\substack{w:\EE[w]=1\\ 0\le w\le 1/\alpha}}\EE[wZ]=\inf_{\tau\in\RR}\left\{
\tau + \frac{1}{\alpha}\EE[(Z-\tau)_+]\right\}=:\CVaR_{1-\alpha}(Z).
\]
\end{proposition}

\begin{proof}
For any feasible $w$ and $\tau\in\RR$, we have
\[
wZ=w\tau+w(Z-\tau)\le w\tau+\frac1\alpha(Z-\tau)_+,
\]
which is due to $0\le w\le 1/\alpha$. Taking expectations and using $\EE[w]=1$,
\[
\EE[wZ]\le\tau+\frac1\alpha\EE[(Z-\tau)_+].
\]
Since this holds for every feasible $w$ and every $\tau$,
\[
\sup_{\substack{w:\EE[w]=1\\ 0\le w\le 1/\alpha}}\EE[wZ]\le\inf_{\tau\in\RR}\left\{\tau+\frac1\alpha\EE[(Z-\tau)_+]\right\}.
\]
For $\alpha=1$, the only feasible weight is $w=1$, and $\tau+\EE[(Z-\tau)_+]=\EE[Z]+\EE[(\tau-Z)_+]\downarrow\EE[Z]$ as $\tau\to-\infty$ by integrability. Thus the identity holds, although a finite minimizing threshold need not exist. For $\alpha\in(0,1)$, take $\tau^*$ to be the $(1-\alpha)$ quantile of $Z$ such that $\Pr(Z>\tau^*)\le\alpha\le\Pr(Z\ge\tau^*)$, and the corresponding $w^*$ is chosen as
\[
w^*=\frac1\alpha\mathbf{1}\{Z>\tau^*\}+\frac\lambda\alpha\mathbf{1}\{Z=\tau^*\},
\]
with $\lambda\in[0,1]$ chosen so that $\EE[w^*]=1$. That is, if we let $q_1=\Pr(Z=\tau^*)$, $q_2=\Pr(Z>\tau^*)$, then $q_2\le\alpha\le q_1+q_2$. Since $q_1,q_2\in[0,1]$, there must exist some $\lambda\in[0,1]$ such that $\lambda q_1+q_2=\alpha$. Using this construction,
\[
\EE[w^*Z]=\frac1\alpha\EE[Z\cdot\mathbf{1}\{Z>\tau^*\}]+\frac\lambda\alpha\EE[Z\cdot\mathbf{1}\{Z=\tau^*\}].
\]
For the first term, we have $Z\cdot \mathbf{1}\{Z>\tau^*\}=(Z-\tau^*)\mathbf{1}\{Z>\tau^*\}+\tau^*\mathbf{1}\{Z>\tau^*\}$. Taking expectation,
\[
\EE[Z\cdot\mathbf{1}\{Z>\tau^*\}]=\EE[(Z-\tau^*)_+]+\tau^*\Pr(Z>\tau^*)=\EE[(Z-\tau^*)_+]+\tau^*q_2.
\]
For the second term, it's easy to see that $\EE[Z\cdot\mathbf{1}\{Z=\tau^*\}]=\tau^*\Pr(Z=\tau^*)=\tau^*q_1$. Combining the two terms,
\[
\EE[w^*Z]=\frac{1}{\alpha}\Big(\tau^*q_2+\EE[(Z-\tau^*)_+]\Big)+\frac\lambda\alpha\tau^*q_1=\frac{1}{\alpha}\EE[(Z-\tau^*)_+]+\frac{\tau^*}{\alpha}(\lambda q_1+q_2)=\frac{1}{\alpha}\EE[(Z-\tau^*)_+]+\tau^*.
\]
Therefore, the proposition holds with exact equality.

\end{proof}

We now give a uniform finite-sample proof. Write $\EE=\EE_{d^D}$ and $\widehat\EE$ for its empirical average. Let $\Theta=\{\theta:\|\theta-\theta_1\|\leq\eta V_\tmax K\}$ for deterministic initialization and step size, and put $B=V_\tmax(1+G)>0$ (the case $V_\tmax=0$ is trivial). Assume finite initial comparator KL. Bounded updates ensure that all iterates lie in $\Theta$. Define the deterministic residual class
\[
\mathcal Z=\{f-f(\cdot,\pi_\theta)-v^\top\nabla\log\pi_\theta:f\in\cF,\ \theta\in\Theta,\ \|v\|\leq V_\tmax\}\subseteq[-B,B]^{\cS\times\cA}.
\]
Assume the classes are measurable and separable in the uniform metrics below, with finite covers. For $N\geq1$, set
\begin{equation}\label{eq:dro-complexity}
r=\frac{B}{4\CC N},\qquad M_Z=\cN_\infty(\mathcal Z,r),\qquad M_\tau=\lceil8\CC N\rceil+1,\qquad
\Comp_D=\log\frac{8M_ZM_\tau}{\delta}.
\end{equation}
Cover centers belong to the relevant classes. For completeness, Assumption~\ref{ass:policy} implies
\[
\|Z_{f,\theta,v}-Z_{\widetilde f,\widetilde\theta,\widetilde v}\|_\infty
\leq2\|f-\widetilde f\|_\infty+V_\tmax(G+\beta)\|\theta-\widetilde\theta\|+G\|v-\widetilde v\|.
\]
Indeed, integrating the score bound along the parameter segment gives $\sup_s\|\pi_\theta(\cdot|s)-\pi_{\widetilde\theta}(\cdot|s)\|_1\leq G\|\theta-\widetilde\theta\|$, which controls the moving expectation $f(s,\pi_\theta)$. Thus, with $r_f=r/6$, $r_\theta=r/[3V_\tmax(G+\beta)]$ and $r_v=r/(3G)$,
\[
M_Z\leq\cN_\infty(\cF,r_f)\cN_{\|\cdot\|}(\Theta,r_\theta)(1+2V_\tmax/r_v)^\dd.
\]
A zero coefficient needs no corresponding cover. This specifies the critic, parameter, update-vector and threshold contributions, including both signs through Eq.~\eqref{eq:dro-complexity}. No separate union bound over rounds is needed.

\begin{lemma}[Uniform DRPU Estimation Error Bound]\label{thm:dro}
Under the preceding boundedness and finite-cover conditions, with probability at least $1-\delta$, simultaneously for all $f,\theta,v$ in the residual class,
\begin{equation}\label{eq:dro-uniform}
|\ell_{f,\theta}(v)-\widehat\ell_{f,\theta}(v)|\leq\Delta_D:=\frac{B}{N}+5B\sqrt{\frac{\CC\Comp_D}{N}}+8B\frac{\CC\Comp_D}{N}.
\end{equation}
Consequently, if $v_k$ is an empirical $\zeta_k$-minimizer over $\|v\|\leq V_\tmax$, then $\ell_k(v_k)-\inf_{\|v\|\leq V_\tmax}\ell_k(v)\leq2\Delta_D+\zeta_k$ simultaneously for all rounds.
\end{lemma}
The empirical loss uses the constraint $\widehat\EE[w]=1$, whereas the population loss uses $\EE[w]=1$. The proof compares these two normalized CVaR functionals directly.

\begin{lemma}[Quantile Characterization of Tail Probability]\label{lemma:quantile}
For any $Z\in\mathcal Z$ and either sign $\sign\in\{\pm1\}$, choose population and empirical minimizing thresholds $\tau^*,\widehat\tau\in[-B,B]$ for $\tau+\CC\EE[(\sign Z-\tau)_+]$ and its empirical counterpart. They satisfy
\[
\Pr(\sign Z>\tau^*)\leq1/\CC\leq\Pr(\sign Z\geq\tau^*),\qquad
\widehat\Pr(\sign Z>\widehat\tau)\leq1/\CC\leq\widehat\Pr(\sign Z\geq\widehat\tau).
\]
\end{lemma}
\begin{proof}
The population subdifferential is $[1-\CC\Pr(\sign Z\geq\tau),\,1-\CC\Pr(\sign Z>\tau)]$ and contains zero at a minimizer; replace $\Pr$ by $\widehat\Pr$ for the empirical statement. Boundedness permits thresholds in $[-B,B]$, including $\tau=-B$ when $\CC=1$. The inequalities handle atoms and ties without requiring equality of tail mass.
\end{proof}
For example, if $z_{(1)}\geq\cdots\geq z_{(N)}$ are the observations of $\sign Z$, let $m=\lfloor N/\CC\rfloor$ and $q=N/\CC-m$. The empirical CVaR is $\frac{\CC}{N}(\sum_{j=1}^{m}z_{(j)}+qz_{(m+1)})$, with the last term omitted when $q=0$. This includes the fractional boundary weight when $N/\CC$ is not an integer.

\begin{proof}[Proof of Lemma~\ref{thm:dro}]
Take an $r$-cover $\{Z_j\}$ of $\mathcal Z$ and a grid $\{\tau_l\}$ of $[-B,B]$ with $M_\tau$ points and spacing at most $r$. The normalized hinges
\[
\mathcal H_0=\left\{\frac{(\sign Z_j-\tau_l)_+}{2B}:\sign\in\{\pm1\},\ j\leq M_Z,\ l\leq M_\tau\right\}
\]
form a deterministic class of at most $2M_ZM_\tau$ functions in $[0,1]$. Lemma~\ref{lemma:empirical-bernstein} applies with $\Gamma=\Comp_D$. For every normalized hinge $H=(\sign Z-\tau)_+/(2B)$, a cover member $H_0$ satisfies
\[
\|H-H_0\|_\infty\leq\nu:=\frac{1}{4\CC N},\qquad \|H^2-H_0^2\|_\infty\leq2\nu.
\]
Thus on a single event of probability at least $1-\delta$, simultaneously for all $Z,\sign,\tau$,
\begin{align*}
|(\EE-\widehat\EE)H|&\leq2\nu+\sqrt{\frac{2(\EE H^2+2\nu)\Comp_D}{N}}+\frac{2\Comp_D}{3N},\\
|(\EE-\widehat\EE)H|&\leq2\nu+2\sqrt{\frac{(\widehat\EE H^2+2\nu)\Comp_D}{N}}+\frac{4\Comp_D}{N}.
\end{align*}
This uniform extension is deterministic, so it remains valid for sample-dependent critics, policies, update vectors and thresholds. At a population minimizer, $\EE H^2\leq1/\CC$; at an empirical minimizer, $\widehat\EE H^2\leq1/\CC$, by Lemma~\ref{lemma:quantile}. Importantly, these are different certificates.

Write $T(Z)=\inf_\tau\{\tau+\CC\EE[(Z-\tau)_+]\}$ and $\widehat T(Z)$ for its empirical counterpart. Optimality gives the two comparisons
\[
T(Z)-\widehat T(Z)\leq\CC(\EE-\widehat\EE)(Z-\widehat\tau)_+,\qquad
\widehat T(Z)-T(Z)\leq\CC(\widehat\EE-\EE)(Z-\tau^*)_+.
\]
Use the empirical certificate for the first and the population certificate for the second. Both absolute deviations are bounded by
\[
4B\CC\nu+4B\sqrt{\frac{\CC\Comp_D}{N}(1+2\CC\nu)}+8B\frac{\CC\Comp_D}{N}
\leq\frac{B}{N}+5B\sqrt{\frac{\CC\Comp_D}{N}}+8B\frac{\CC\Comp_D}{N}.
\]
The same holds for $-Z$. Since $\ell(v)=\max\{T(Z),T(-Z)\}$, the maximum is 1-Lipschitz in its two arguments and Eq.~\eqref{eq:dro-uniform} follows. Finally, for any feasible $v$, empirical near-optimality yields $\ell_k(v_k)\leq\widehat\ell_k(v_k)+\Delta_D\leq\widehat\ell_k(v)+\zeta_k+\Delta_D\leq\ell_k(v)+\zeta_k+2\Delta_D$. Taking the infimum proves the excess-loss bound.
\end{proof}

\begin{proof}[Proof of Theorem~\ref{thm:dro-guarantee}]
Fix a chosen radius $\CC\geq1$ satisfying Assumption~\ref{ass:data}. The density ratio lies in $\cW_\infty$, so $\err_k\leq\ell_k(v_k)\leq\widetilde\epsilon_\bias+2\Delta_D+\zeta_k$. With finite initial comparator KL, the same KL-potential calculation as in Lemma~\ref{lemma:NPG-regret} gives, for any fixed $\eta>0$, on the uniform event of Lemma~\ref{thm:dro},
\begin{equation}\label{eq:dro-full}
\frac{\Reg_K}{K}\leq\frac{\KL(\pi_\cp\|\pi_1)}{\eta K}+\frac{\beta\eta V_\tmax^2}{2}+\widetilde\epsilon_\bias+2\Delta_D+\frac1K\sum_{k=1}^K\zeta_k.
\end{equation}
For exact minimization and the step size of Lemma~\ref{lemma:NPG-regret} with $B_L=V_\tmax$, this proves the main-text theorem. Only when $N\geq\CC\Comp_D$ may the linear term be absorbed into the square-root term. Without this regime the full bound remains valid. The bound need not vanish at finite $N$ when $\widetilde\epsilon_\bias=0$; this is an upper-bound statement, not a claim that DRPU necessarily incurs nonzero estimation error.
\end{proof}

\subsection{Computation}
\label{app:dro-computation}
In Section~\ref{sec:dro} we mentioned that minimizing the loss $\hat{\ell}_k(v)$ can be viewed as a linear program (or more generally, a SOCP)\footnote{The specific type of the program depends on the norm constraint of $v: \|v\|\leq V_\tmax$. That is, if it's $\|\cdot\|_\infty$ or $\|\cdot\|_1$, then it's a linear program (LP); if it's $\|\cdot\|_2$, then it's a second-order cone program (SOCP)}, which can be efficiently solved by any convex solver. Now we give the specific realization of this program. Recall the CVaR expression of $\hat{\ell}_k$ is given by
\[
\hat{\ell}_k(v)=\max\left\{\min_{\tau\in\RR}\left\{\tau+\frac{\CC}{N}\sum_{i=1}^N \left(\epsilon_v^{(i)}-\tau\right)_+\right\},\min_{\tau\in\RR}\left\{\tau+\frac{\CC}{N}\sum_{i=1}^N \left(-\epsilon_v^{(i)}-\tau\right)_+\right\}\right\},
\]
where $\epsilon_v^{(i)}=A_k^{(i)}-v^\top\nabla\log\pi_k(a^{(i)}|s^{(i)})$. Since the outer max of the two CVaRs with an epigraph variable (similarly for the inner $\tau$, we can write down the corresponding $\dd$-dimensional convex program of $\min_{v\in\RR^\dd:\|v\|\leq V_\tmax}\hat{\ell}_k(v)$:
\begin{align*}
\min_{\substack{v, z, \tau_{+}, \tau_{-},\\ \{c_{+}^{(i)}\}_{i=1}^N, \{c_{-}^{(i)}\}_{i=1}^N}}
\quad & z \\[2pt]
\text{s.t.}\quad
& z \ge \tau_{+} + \frac{\CC}{N}\sum_{i=1}^N c_{+}^{(i)}, \\[2pt]
& c_{+}^{(i)} \ge A_k^{(i)} - v^\top \phi_k^{(i)} - \tau_{+}, \qquad
  c_{+}^{(i)} \ge 0, \quad i=1,\dots,N, \\[6pt]
& z \ge \tau_{-} + \frac{\CC}{N}\sum_{i=1}^N c_{-}^{(i)}, \\[2pt]
& c_{-}^{(i)} \ge -A_k^{(i)} + v^\top \phi_k^{(i)} - \tau_{-}, \qquad
  c_{-}^{(i)} \ge 0, \quad i=1,\dots,N, \\[6pt]
& \|v\| \le V_{\max}.
\end{align*}

\subsection{Feature Coverage and Chi-Square Weight Class}\label{app:dro-w2}

In this section, we discuss the feature coverage condition in Section~\ref{sec:npg} and clarify its relation to the chi-square weight class introduced in Section~\ref{sec:dro}.

\paragraph{Feature coverage condition.}
Recall that in the compatible case, the advantage function $A_k(s,a)$ lies in the linear span of the policy-gradient \emph{features}, which we denote by $\phi_k(s,a)=\nabla\log\pi_k(a\mid s)$. Let $\Sigma_D^\pi=\EE_{d^D}[(\nabla\log\pi)(\nabla\log\pi)^\top]$ denote the feature covariance matrix under the data distribution $d^D$ (suppose it's invertible), and let $\mu_\cp^{\pi}=\EE_{(s,a)\sim d^{\pi_\cp}}[\nabla\log\pi(a\mid s)]$ denote the mean feature under $d^{\pi_\cp}$. Then the feature coverage condition $\CC_\feat$ can be written as
\[
(\mu_\cp^{\pi_k})^\top(\Sigma_D^{\pi_k})^{-1}\mu_\cp^{\pi_k}\le \CC_\feat,\qquad \forall k\in[K].
\]
Unlike density coverage, which controls distribution shift uniformly over all $(s,a)$ pairs, feature coverage only controls transfer on the restricted function class induced by the policy-gradient features. That is, $\CC_\feat$ requires data only to cover a single direction $\mu_\cp^{\pi_k}$, which is the mean feature direction under $d^{\pi_\cp}$. Since the CFA error is exactly a linear error in the feature class, the feature coverage is the known tightest coverage notion under linear function approximation\citep{jiang2025offline}.

The next proposition shows that feature coverage is always implied by density coverage, and can therefore be viewed as a refined notion of coverage.

\begin{proposition}[Density Coverage Implies Feature Coverage]\label{prop:density-implies-feature-coverage}
Assume Assumption~\ref{ass:data} holds, namely $\|w^*\|_\infty\le \CC$ for $w^*=d^{\pi_\cp}/d^D$. Then for every $k\in[K]$ and every $f\in\spn(\phi_k)$,
\[
\big(\EE_{d^{\pi_\cp}}[f]\big)^2\le \CC\,\EE_{d^D}[f^2].
\]
Consequently, for invertible $\Sigma_D^{\pi_k}$, this is equivalent to $
(\mu_\cp^{\pi_k})^\top(\Sigma_D^{\pi_k})^{-1}\mu_\cp^{\pi_k}\le \CC$. In particular, the feature coverage constant always satisfies $\CC_\feat\le \CC$.
\end{proposition}

\begin{proof}
Fix any $k\in[K]$ and any $f\in\spn(\phi_k)$. Since $w^*=d^{\pi_\cp}/d^D$, we have $\EE_{d^{\pi_\cp}}[f]=\EE_{d^D}[w^*f]$. Therefore,
\[
\big(\EE_{d^{\pi_\cp}}[f]\big)^2=\big(\EE_{d^D}[w^*f]\big)^2\le \EE_{d^D}[(w^*)^2]\EE_{d^D}[f^2]\le \CC\,\EE_{d^D}[f^2],\qquad \forall f\in\spn(\phi_k),
\]
where the first inequality is by Cauchy-Schwarz, and the second inequality leverages the ``low-variance'' property of the importance weight. That is, for $w^*$ such that $\EE_{d^D}[w^*]=1$ and $\|w^*\|_\infty\le \CC$, we have $\EE_{d^D}[(w^*)^2]\le \|w^*\|_\infty\,\EE_{d^D}[w^*]\le \CC$.

Now write any $f\in\spn(\phi_k)$ as $f_\theta(s,a)=\theta^\top\phi_k(s,a)$ for some $\theta\in\mathbb{R}^d$. Then $\EE_{d^{\pi_\cp}}[f_\theta]=\theta^\top\mu_\cp^{\pi_k}$ and $\EE_{d^D}[f_\theta^2]=\theta^\top\Sigma_D^{\pi_k}\theta$. Hence,
\[
\sup_{\theta\neq 0}\frac{(\theta^\top\mu_\cp^{\pi_k})^2}{\theta^\top\Sigma_D^{\pi_k}\theta}\le \CC,\quad\forall k\in[K].
\]
Since $\Sigma_D^{\pi_k}$ is invertible, the left-hand side is the Rayleigh quotient
\[
\sup_{\theta\neq 0}\frac{(\theta^\top\mu_\cp^{\pi_k})^2}{\theta^\top\Sigma_D^{\pi_k}\theta}=(\mu_\cp^{\pi_k})^\top(\Sigma_D^{\pi_k})^{-1}\mu_\cp^{\pi_k},
\]
which implies
\[
(\mu_\cp^{\pi_k})^\top(\Sigma_D^{\pi_k})^{-1}\mu_\cp^{\pi_k}\le \CC,\quad\forall k\in[K].
\]
Taking the supremum over $k\in[K]$ gives $\CC_\feat\le \CC$.
\end{proof}

Proposition~\ref{prop:density-implies-feature-coverage} shows that feature coverage is never worse than density coverage. More importantly, it can be \emph{substantially smaller}. The reason is that density coverage controls the worst-case reweighting error over the entire state-action space, whereas feature coverage only measures mismatch along the specific feature directions relevant to the actor update.

\paragraph{Chi-square weight class.}
We now connect the above discussion to the chi-square weight class in Section~\ref{sec:dro}. \emph{In the compatible linear case, one need not use the true density ratio for distribution transfer}. Instead, it suffices to use a feature-based correction function that is exact on the linear residual class. Define
\begin{equation}\label{eq:feature-correction}
w_k^*(s,a):=(\mu_\cp^{\pi_k})^\top(\Sigma_D^{\pi_k})^{-1}\phi_k(s,a),
\end{equation}
and consider the chi-square weight class
\[
\cW_{\chi^2}:=\{w:\EE_{d^D}[w^2]\le \CC_\feat\}.
\]

\begin{proposition}\label{prop:feature-weight-class}
Assume the compatible case, so that $A_k(s,a)\in\spn(\phi_k)$. Then the feature-based correction weight $w_k^*$ defined above satisfies the following two properties:

(i) (\textbf{Distribution transfer}) for any $f\in\spn(\phi_k)$, $\EE_{d^{\pi_\cp}}[f]=\EE_{d^D}[w_k^*f]$.

(ii) (\textbf{Realizability}) $\EE_{d^D}[(w_k^*)^2]=(\mu_\cp^{\pi_k})^\top(\Sigma_D^{\pi_k})^{-1}\mu_\cp^{\pi_k}\le \CC_\feat$, and hence $w_k^*\in\cW_{\chi^2}$.
\end{proposition}

\begin{proof}
For any $f\in\spn(\phi_k)$, there exists $\theta\in\mathbb{R}^d$ such that $f(s,a)=\theta^\top\phi_k(s,a)$. Then
\[
\EE_{d^D}[w_k^*f]=(\mu_\cp^{\pi_k})^\top(\Sigma_D^{\pi_k})^{-1}\EE_{d^D}[\phi_k\phi_k^\top]\theta=(\mu_\cp^{\pi_k})^\top\theta=\EE_{d^{\pi_\cp}}[f],
\]
which proves (i). For (ii),
\[
\EE_{d^D}[(w_k^*)^2]=(\mu_\cp^{\pi_k})^\top(\Sigma_D^{\pi_k})^{-1}\EE_{d^D}[\phi_k\phi_k^\top](\Sigma_D^{\pi_k})^{-1}\mu_\cp^{\pi_k}=(\mu_\cp^{\pi_k})^\top(\Sigma_D^{\pi_k})^{-1}\mu_\cp^{\pi_k}\le \CC_\feat.
\]
Thus $w_k^*\in\cW_{\chi^2}$.
\end{proof}

Proposition~\ref{prop:feature-weight-class} shows that, in the compatible linear case, the chi-square weight class can be instantiated directly by feature coverage: the correction function $w_k^*$ exactly transfers expectations on the linear residual class, and its second moment under $d^D$ is precisely the feature coverage quantity. In this sense, $\cW_{\chi^2}$ provides a feature-based analogue of the density-ratio weight class that is tailored to the actor update.

\paragraph{Remark: feature correction weight.} It is important to note that the feature-based correction weight $w_k^*$ in Eq.~\eqref{eq:feature-correction} should NOT be interpreted as a density ratio. In general, it need not be nonnegative and need not satisfy the normalization condition $\EE_{d^D}[w_k^*]=1$. This is not an issue for our purpose, since the role of $w_k^*$ is only to realize exact transfer on the residual class $\spn(\phi_k)$, rather than to represent a valid change of measure.

The normalization issue can also be handled by a simple centering trick. Let $\mu_D^{\pi_k}=\EE_{d^D}[\phi_k]$, $\mu_\cp^{\pi_k}=\EE_{d^{\pi_\cp}}[\phi_k]$, and $\bar\Sigma_D^{\pi_k}=\EE_{d^D}[(\phi_k-\mu_D^{\pi_k})(\phi_k-\mu_D^{\pi_k})^\top]$ (assume invertible). Define
\[
\bar w_k(s,a)=1+(\mu_\cp^{\pi_k}-\mu_D^{\pi_k})^\top(\bar\Sigma_D^{\pi_k})^{-1}(\phi_k(s,a)-\mu_D^{\pi_k}).
\]
Then $\EE_{d^D}[\bar w_k]=1$, and the same exact transfer property holds for every $f\in\spn(\phi_k)$. Moreover,
\[
\EE_{d^D}[(\bar w_k-1)^2]=(\mu_\cp^{\pi_k}-\mu_D^{\pi_k})^\top(\bar\Sigma_D^{\pi_k})^{-1}(\mu_\cp^{\pi_k}-\mu_D^{\pi_k}).
\]
Thus, normalization can be kept by replacing the uncentered 2nd-moment geometry with a centered covariance geometry. Nonnegativity, however, is a separate requirement and is not needed for restricted transfer over the linear residual class.

This observation highlights a broader distinction between our correction-weight perspective and the traditional DRO or importance-weighting literature. In classical DRO \citep{kuhn2025distributionally,rahimian2019distributionally}, the uncertainty set $\cW$ is typically chosen as an $f$-divergence ball around the data distribution $d^D$, so its elements correspond to valid probability reweightings. Similarly, importance weighting methods, including marginalized importance sampling (MIS) \citep{liu2018breaking,nachum2019dualdice}, use correction weights that are valid density ratios. In contrast, our feature-based correction weights need not define a valid change of measure; they only need to realize the desired transfer over the function class relevant to the actor update. In this sense, correction weights are more general than density ratios.

\section{Omitted Details for Section~\ref{sec:comparison}}

\subsection{Recovery of Behavior Cloning} \label{app:bc}
Recall that with known $d^{\pi_\cp}$, let $m_k=\EE_{(s,a)\sim d^{\pi_\cp}}[A_k(s,a)]$ and $\mu_k=\EE_{(s,a)\sim d^{\pi_\cp}}[\nabla\log\pi_k(a\mid s)]$, we solve for $\min_{v:\|v\|\leq V_\tmax}|m_k-v^\top\mu_k|$ in Eq.~\eqref{eq:mm}. There exists a closed-form solution which can be expressed as
\begin{equation}\label{eq:mm-update}
v_k=\begin{cases}
\frac{m_k}{\|\mu_k\|_*}u_k,&\text{if}\quad |m_k|\leq V_\tmax\cdot\|\mu_k\|_*,\\
V_\tmax\cdot u_k,&\text{if}\quad m_k>V_\tmax\cdot\|\mu_k\|_*,\\
-V_\tmax\cdot u_k,&\text{if}\quad m_k<-V_\tmax\cdot\|\mu_k\|_*,
\end{cases}
\end{equation}
where $u_k=\arg\max_{u:\|u\|\leq 1}u^\top\mu_k$ aligns with the direction of $\mu_k$. This can be viewed as running steepest descent on $\Phi(\theta)$, defined as the expected KL divergence under $d^{\pi_\cp}$:
\[
\Phi(\theta)=\EE_{s\sim d^{\pi_\cp}}[\KL(\pi_\cp(\cdot\mid s)\|\pi_\theta(\cdot\mid s))].
\]
This is because that the negative gradient of $\Phi$ corresponds to feature mean (at round $k$) $\mu_k$:
\begin{align*}
\nabla_\theta\Phi(\theta_k)&=\nabla_\theta\EE_{s\sim d^{\pi_\cp}}\left[\EE_{a\sim \pi_\cp(\cdot\mid s)}\left[\log\frac{\pi_\cp(a\mid s)}{\pi_k(a\mid s)}\right]\right]\\
&=\nabla_\theta\EE_{(s,a)\sim d^{\pi_\cp}}[\log\pi_\cp(a\mid s)-\log\pi_k(a\mid s)]\\
&=-\EE_{(s,a)\sim d^{\pi_\cp}}[\nabla_\theta\log\pi_k(a\mid s)]=-\mu_k.
\end{align*}

\paragraph{Sign of advantage.} There is a subtlety regarding the sign of the advantage term: we are only updating along the direction of $-\nabla\Phi(\theta_k)$ to minimize the BC objective if $m_k=\EE_{d^{\pi_\cp}}[A_k]\ge 0$; 
this is the BC regime described in Section~\ref{sec:comparison}. However, when $m_k<0$: the update will \textit{increase} $\Phi(\theta)$ and move away from the comparator. To understand the condition $m_k \ge 0$, consider the case where $f_k$ is accurate, i.e., $f_k=Q^{\pi_k}$, then
\[
m_k\ge 0\quad\Longleftrightarrow\quad \EE_{s\sim d^{\pi_\cp}}[A^{\pi_k}(s,a)]\ge 0\quad\Longleftrightarrow\quad J(\pi_\cp)\ge J(\pi_k),
\]
where the second equivalence is due to the performance-difference lemma (PDL). This means $m_k\ge 0$ corresponds to $J(\pi_k)\le J(\pi_\cp)$, whereas $m_k<0$ indicates that the current policy is already outperforming the comparator (at least from the critic's perspective). Therefore the update deliberately moves towards an ``improving'' direction which depends on whether the current policy is better than the comparator $\pi_\cp$. This means that DRPU remains aligned with the BC descent direction in typical settings when the learned policy is unlikely to exceed the performance of $\pi_\cp$. 

Finally, this sign issue does not conflict with our theoretical guarantees. The reason is that the guarantees in the paper are stated in terms of comparator regret, not monotonic decrease of the KL objective. Recall that in Section~\ref{sec:pspi} we introduce the actor-side quantity as
\[
\frac{\Reg_K}{K}=\frac1K\sum_{k=1}^K\EE_{s\sim d^{\pi_\cp}}[f_k(s,\pi_\cp)-f_k(s,\pi_k)]=\frac1K\sum_{k=1}^Km_k,
\]
which is precisely the average advantage with respect to the comparator policy. Thus, rounds with $m_k<0$ are not harmful for the regret analysis; if anything, they help, because they indicate that the current policy already outperforms the comparator. What matters in the analysis is not the iterates remain close to $\pi_\cp$, but rather that the mean-matching error is controlled, which is where the norm constraint $\|v_k\|\le V_\tmax$ plays an important stabilizing role here.

\subsection{Numerical Result Setting}
\label{app:experiment}

We consider a simple discounted MDP with finite state and action spaces. The state space is $\cS=\{1,2,3\}$ and the action space is $\cA=\{a_1,a_2\}$. The discount factor is $\gamma=0.9$, and the initial state distribution $d_0$ is uniform over $\cS$. The transition dynamics are state-absorbing, so the state does not change over time. As a result, the problem reduces to a contextual bandit with discounted returns, while still admitting a well-defined occupancy measure $d^\pi$.

The reward function is deterministic and state-dependent. For action $a_1$, the reward is given by $r(s,a_1)=(1,4,4)$ for states $s=1,2,3$, respectively. For action $a_2$, the reward is constant across states, given by $r(s,a_2)=(2,2,2)$. This construction ensures that the advantage function varies across states and actions, while remaining smooth and bounded. In particular, the policy that always selects $a_1$ is not optimal, since action $a_2$ yields higher reward in state $s=1$.

We consider a one-dimensional softmax policy class $\Pi_\theta=\{\pi_\theta:\theta\in\RR\}$ parameterized by
\[
\pi_\theta(a_1\mid s)=\frac{\exp(\theta c_s)}{\exp(\theta c_s)+\exp(-\theta c_s)},\qquad
\pi_\theta(a_2\mid s)=1-\pi_\theta(a_1\mid s),
\]
where the state-dependent coefficients are $c=(1,2,3)$. The corresponding score function $\nabla\log\pi_\theta(a\mid s)$ is bounded and one-dimensional, which makes model misspecification effects transparent.

The comparator policy $\pi_\cp$ is chosen as the deterministic policy that always selects action $a_1$ in every state. Importantly, $\pi_\cp$ is \emph{not} an optimal policy for this MDP. In the experiments, $\pi_\cp$ is approximated within the policy class by a large parameter value $\theta_\cp=100$. The data distribution is set to $d^D=d^{\pi_\cp}$, corresponding to the no-shift setting as in imitation learning with expert-generated data.

All methods are initialized at $\theta_0=0$ and run for a fixed number of iterations. Performance is evaluated by tracking the policy value and the CFA error $\err_k$ under $d^{\pi_\cp}$ at each iteration. Since $d^D=d^{\pi_\cp}$, there is no distribution shift in this experiment, and any observed performance gap reflects model misspecification alone.

\section{Technical Lemmas}

\begin{lemma}[Gibbs Variational Principle]\label{lemma:gibbs}
For any measurable function $\phi:\cA\to\RR$ and any distribution $u\in\Delta_\nu(\cA)$ that is absolutely continuous w.r.t. the base measure $\nu$, we have
\[
-\frac{1}{\eta}\log\int_\cA \exp(-\eta \phi(a))\nu(\mathrm{d}a)
= \inf_{u\in\Delta_\nu(\cA)}\left\{\EE_{a\sim u}[\phi(a)] + \frac{1}{\eta}\KL(u\| \nu)\right\}.
\]
Moreover, the infimum is attained at the Gibbs (softmax) distribution
\[
u^\star(a) = \frac{\exp(-\eta \phi(a))}{\int_\cA \exp(-\eta \phi(a'))\nu(da')}.
\]
\end{lemma}

\begin{lemma}[Bellman Error Telescoping{\citep[Lemma 2]{jiang2025offline}}]\label{lemma:BE-telescoping}
For any $\pi:\cS\to\Delta(\cA)$, and any $f\in\RR^{\cS\times\cA}$, the performance gap using $f$ as an estimate of $Q^\pi$ is given by
\[
J_f(\pi)-J(\pi)=\frac{1}{1-\gamma}\EE_{d^\pi}[f-\cT^\pi f],
\]
where $J_f(\pi)=\EE_{s\sim d_0}[f(s,\pi)]$, and $d^\pi$ is the discounted state-action occupancy of $\pi$.
\end{lemma}

\begin{lemma}[Generalized Performance-Difference Lemma {\citep[Lemma 6]{jiang2025offline}}]\label{lemma:PDL}
For any $f\in\RR^{\cS\times\cA}$, and policies $\pi,\pi':\cS\to\Delta(\cA)$, the difference between their expected return is given by
\[
J(\pi')-J(\pi)=\frac{1}{1-\gamma}\left(\EE_{s\sim d^{\pi'}}\left[f(s,\pi')-f(s,\pi)\right]+\EE_{d^{\pi'}}[\cT^\pi f-f]+\EE_{d^\pi}[f-\cT^\pi f]\right).
\]
\end{lemma}


\begin{lemma}[Matrix Chernoff Bound {\citep[Corollary 5.2]{tropp2012user}}]\label{lemma:matrix-chernoff}
Let $X_1,\ldots,X_n$ be independent random vectors in $\RR^\dd$ with $\EE[X_iX_i^\top]=I$ and $\|X_i\|_2\leq b$ almost surely. For $t\in(0,1)$,
\[
\Pr\left\{\lambda_\tmin\left(\frac1n\sum_{i=1}^nX_iX_i^\top\right)<1-t\right\}\leq\dd\exp\left(-\frac{nt^2}{2b^2}\right).
\]
\end{lemma}

\begin{lemma}[Freedman Inequality for Vector-Valued Martingales {\citep[Lemma 15]{hsu2012random}}]
\label{lemma:matrix-bernstein}
Let $X_1, \ldots, X_n$ be a martingale difference vector sequence in a Euclidean space
(i.e., $\EE[X_i \mid X_{1:i-1}] = 0$ for all $i=1,\ldots,n$)
such that for all $i$,
\[
\sum_{i=1}^n \EE\big[\|X_i\|^2 \mid X_{1:i-1}\big] \le v,\qquad
\|X_i\| \le b,
\]
almost surely. Then for any $\delta \in (0,1)$, with probability at least $1-\delta$,
\[
\Big\|\sum_{i=1}^n X_i\Big\|\le\sqrt{v}\left(1 + \sqrt{8\log\frac{1}{\delta}}\right)
+\frac{4}{3}b\log\frac{1}{\delta}.
\]
\end{lemma}




\begin{lemma}[Uniform Empirical-Second-Moment Bernstein Bound]\label{lemma:empirical-bernstein}
Let $\mathcal H_0$ be a deterministic finite class of $m$ measurable functions into $[0,1]$, evaluated on $N$ i.i.d. observations. Put $\Gamma=\log(4m/\delta)$ and $a=\Gamma/N$. With probability at least $1-\delta$, simultaneously for all $h\in\mathcal H_0$,
\[
|(\EE-\widehat\EE)h|\leq\sqrt{2(\EE h^2)a}+2a/3,\qquad
|(\EE-\widehat\EE)h|\leq2\sqrt{(\widehat\EE h^2)a}+4a.
\]
\end{lemma}
\begin{proof}
Two-sided scalar Bernstein applied to $h$ and $h^2$, followed by a union bound over $m$ functions, gives the first bound and $q:=\EE h^2\leq\widehat\EE h^2+\sqrt{2qa}+2a/3$, since $\operatorname{Var}(h^2)\leq\EE h^4\leq q$. Using $\sqrt{2qa}\leq q/2+a$ gives $q\leq2\widehat\EE h^2+10a/3$. Substitution into the first bound proves the second, because $\sqrt{20/3}+2/3<4$. There are at most $4m$ one-sided failure events, each of probability at most $e^{-\Gamma}$.
\end{proof}


\end{document}